\documentclass[11pt]{article}

\usepackage{graphicx}
\usepackage{amssymb}
\usepackage{amsthm,amsmath,dsfont}

\usepackage{lineno}

\usepackage[colorlinks=true, citecolor=blue]{hyperref}
\usepackage{algorithm2e}
\usepackage{algcompatible}
\usepackage{wrapfig}
\usepackage{natbib}
\usepackage{comment}
\usepackage{longtable}

\usepackage{float}
\usepackage{subfig}
\usepackage{wrapfig}
\usepackage{framed,xcolor}
\usepackage{newfloat}
\usepackage[xcolor,framemethod=TikZ]{mdframed}
\usepackage{mathpazo}
\usepackage{amsthm}
\usepackage{thmtools}
\usepackage{authblk}
\patchcmd{\thebibliography}{\section*{\refname}}{}{}{}

\newcommand{\RR}{\mathbb{R}}

\newcommand{\EE}{\mathbb{E}}
\newcommand{\PP}{\mathbb{P}}

\newcommand{\dd}{\mathrm{d}}

\newcommand{\supp}{\mathsf{supp}}

\newtheoremstyle{theoremdd}
{\topsep}
{\topsep}
{\itshape}
{0pt}
{\fontfamily{cmss}\selectfont\bfseries}
{.}
{ }
{\thmname{#1}\thmnumber{ #2}\thmnote{ (#3)}}

\theoremstyle{theoremdd}

\newtheorem{proposition}{Proposition}

\mdfdefinestyle{myStyle}{roundcorner=5pt,backgroundcolor=brown!20,linecolor=red!40!black,linewidth=1.5pt}

\usepackage{titlesec}
\titleformat*{\section}{\fontfamily{cmss}\selectfont\large\bfseries}
\titleformat*{\subsection}{\fontfamily{cmss}\selectfont\normalsize\bfseries}
\titleformat*{\subsubsection}{\fontfamily{cmss}\selectfont\normalsize}
\usepackage[left=1.3 in, right=1.3 in, top=1.25 in, bottom=1.25 in]{geometry}
\graphicspath{{./figures/}}

\begin{document}
	
		
		
		\title{\fontfamily{cmss}\selectfont Sparse Travel Time Estimation from Streaming Data}
		


\author[1,2]{Saif Eddin Jabari\footnote{Corresponding author, E-mail: \url{sej7@nyu.edu}}}
\author[3]{Nikolaos M. Freris}	
\author[4]{Deepthi Mary Dilip}
\affil[1]{New York University Abu Dhabi, Saadiyat Island, P.O. Box 129188, Abu Dhabi, U.A.E.}
\affil[2]{New York University Tandon School of Engineering, Brooklyn NY}
\affil[3]{School of Computer Science and Technology, University of Science and Technology of China, Diansan Building, West Campus, 443 Huangshan Road, Hefei, Anhui, 230027 China}
\affil[4]{BITS Pilani, Dubai Campus, Dubai International Academic City, Dubai, U.A.E.}

\date{}
		
\maketitle	

{ \fontfamily{cmss}\selectfont\large\bfseries		
\begin{abstract}
{ \normalfont\normalsize
	We address two shortcomings in online travel time estimation methods for congested urban traffic. The first shortcoming is related to the determination of the number of mixture modes, which can change dynamically, within day and from day to day.  The second shortcoming is the wide-spread use of Gaussian probability densities as mixture components.  Gaussian densities fail to capture the positive skew in travel time distributions and, consequently, large numbers of mixture components are needed for reasonable fitting accuracy when applied as mixture components.  They also assign positive probabilities to negative travel times.  To address these issues, this paper derives a mixture distribution with Gamma component densities, which are asymmetric and supported on the positive numbers.  We use sparse estimation techniques to ensure parsimonious models and propose a generalization of Gamma mixture densities using Mittag-Leffler functions, which provides enhanced fitting flexibility and improved parsimony. 
	In order to accommodate within-day variability and allow for online implementation of the proposed methodology (i.e., fast computations on streaming travel time data), we introduce a recursive algorithm which efficiently updates the fitted distribution whenever new data become available.  Experimental results using real-world travel time data illustrate the efficacy of the proposed methods.
	
	\medskip
	
	\textbf{\fontfamily{cmss}\selectfont Keywords}: Multi-modal travel time distributions; sparse modeling; Mittag-Leffler density; data-driven traffic analytics.
}
\end{abstract}
}
		
		
	
	
	

\section{Introduction}
\label{S:intro}

Travel times are among the prime measures of traffic and travel performance in congested road networks. They are critical inputs in a variety of route planning applications and can vary dramatically from one location to another and by time of day.  This variability is a key factor in assessing the reliability of traffic routes. Many factors contribute to travel time variability including uncertainty about network supplies and demands, driver behavior, and queueing dynamics at traffic signals \citep{du2012adaptive,ramezani2012estimation,ramezani2015queue}.  

Travel time modeling and estimation remains an active area of research in Transportation Science: \cite{kharoufeh2004deriving} derived a stochastic model for travel time on a freeway link. \cite{carey2005alternative} established conditions for a well-behaved travel time model, and further proposed a discretized model and analyzed its convergence properties to the celebrated LWR model \citep{carey2005convergence}. \cite{ghiani2014note} showed that any continuous piecewise-linear travel time model can be generated by an appropriate Ichoua, Gendreau, Potvin (IGP) model \citep{ichoua2003vehicle}, and provided an efficient method for learning the parameters via solving a linear system of equations.  \cite{gomez2016modeling} introduced a model for vehicle routing problems with stochastic travel and service times (VRPSTT).  \cite{zheng2017methodological} developed an analytical model that captures travel time dynamics and variability in urban signalized arterials.

We present a \textbf{data-driven} methodology for learning travel times modeled as a mixture of densities.  This approach is well suited to congested urban networks, where trip information (times and positions) is only available for the vehicle transmitting the information (e.g., from taxis or car-sharing service providers) and information about other traffic characteristics (such as traffic volumes, speeds of other vehicles, and traffic control settings) is not available.  The goal is to estimate travel time distributions in this (common) type of setting.

Along expressways (uninterrupted traffic facilities), travel time distributions are typically well captured by unimodal functions such as the lognormal distribution \citep{richardson1978travel,rakha2006estimating,pu2011analytic,arezoumandi2011estimation}, the Gamma distribution \citep{polus1979study,kim2014finite,kim2015compound} and the Weibull distribution \citep{al2006new}.  In urban settings with traffic signals, travel time distributions tend to have multiple modes.  Along high-speed arterials and expressways with stop-and-go traffic, they are well represented by bi-modal distributions \citep{hofleitner2012arterial,kazagli2013arterial,ji2013travel,feng2014probe}.  In congested networks with spillover dynamics, one tends to observe more than two modes \citep{rakha2011feasibility,hofleitner2012probability,hunter2013large,yang2014travel}.  To account for this multi-modality of travel time distributions, researchers resort more and more to  mixture modeling \citep{guo2010multistate,wan2014prediction,rahmani2015non}.  The Expectation Maximization (EM) algorithm \citep{redner1984mixture} and Bayesian techniques are widely used to estimate mixture model parameters. The Bayesian approach applies Markov chain Monte Carlo (MCMC) techniques to solve the estimation problem and is known to be computationally demanding \citep{chen2014application}. As a result, the majority of prevalent methods utilize the EM algorithm, which is most suitable for Gaussian mixtures.  The concern with computation times stems from a need for real-time estimation.

The symmetric shape of Gaussian densities is in contrast to the conventional (and empirically supported) representation of travel time distributions using distributions with positive skew \citep{emam2006using,fosgerau2012valuing,jenelius2013travel,xu2014modeling,kim2015compound,jenelius2015probe,taylor2017fosgerau}.  When the underlying distributions are skewed and the mixture components are not, a large number of components is needed for accurate estimation of travel time distributions.  This can adversely impact parsimony of the model.  Another disadvantage that comes with adopting Gaussian mixture components is that the resulting probability distribution has \emph{negative travel times in its support}.  This feature is unavoidable and particularly problematic for travel time estimation over short segments that have high variability.

Traditional mixture modeling requires a priori knowledge of the number of mixture components.  This is a major limitation in the context of travel time estimation, since the number of components and their parameters changes throughout the day.  The problem of determining the optimal number of components has been addressed by researchers in various fields through sparse density estimators using support vector machines \citep{weston19991}, penalized histogram difference criteria \citep{lin2013evaluate}, and orthogonal forward regression \citep{chen2004sparse,chen2008orthogonal}.  In the transportation literature, the number of components is typically determined by seeking sparse (i.e., parsimonious) solutions to problems with large numbers of candidate components \citep{hofleitner2013online,hofleitner2014learning}.  

This paper proposes a mixture density estimation approach for real-time estimation of travel time distributions. Our analysis substantially extends and expands our previous work \citep{dilip2017sparse}.  We derive a suitable form for the mixture distribution from fundamentals of macroscopic traffic theory. The source of uncertainty about travel times can be interpreted as an absence of knowledge about detailed traffic conditions along the travel routes in question, which is represented by random traffic density profiles.  We demonstrate that, for any equilibrium pace function (defined as the reciprocal of an equilibrium speed relation), the distribution of travel times can be captured by a mixture of Gamma probability density functions.  The parameters of the mixture component densities can be tuned a priori, whence the estimation problem focuses on finding estimates for the mixture weights.  The Gamma probability densities overcome the issues mentioned above pertaining to Gaussian mixtures, since the Gamma component densities have positive support and can have asymmetric shapes (i.e., they are more flexible).

To further enhance parsimony, we devise a richer set of component densities (with variable location \emph{and} scale parameters) that can be combined to capture a wide variety of travel time distributions. This is achieved by generalizing the Gamma densities using Mittag-Leffler functions.
Subsequently, the problem becomes one of choosing those mixture components that most closely capture the empirical distributions.  We propose the use of an $\ell_1-$regularizer, which is known to promote sparsity \citep{tibshirani1996regression} and demonstrate how to apply this methodology to streaming data.  The latter is achieved by (i) updating the inputs whenever a new travel time sample or batch of travel times arrives and (ii)  warm-starting the numerical optimization; this allows for a very fast update of the fitted distribution and renders the proposed approach amenable to an online implementation capable of capturing within-day variation of travel times. 

The remainder of this paper is organized as follows: We derive the Gamma mixture from traffic flow fundamentals in \autoref{sec:TTdistDerivation}. In \autoref{sec:sparseDenEst}, we formulate the estimation problem and describe a discretization procedure that casts it as a convex program. . We specialize the estimation problem formulation (specifically, the discretization) to Gamma mixture components and present our proposed generalization using Mittag-Leffler functions in \autoref{sec:kernelChoice}.  \autoref{sec:numerics} describes the numerical optimization approach used to solve the estimation problem, while recursive estimation from streaming data is discussed in \autoref{sec:recursive}.  \autoref{sec:testing} is devoted to testing of the proposed approach using both synthetic data (for validation) and real-world data (for demonstrating the applicability in real-life settings), while \autoref{sec:conclusion} concludes the paper. Our findings firmly support the efficacy of the proposed sparse density estimation for online learning of travel time distributions, in terms of improved fitting accuracy and improved parsimony.

We also provide five appendices that support the main sections of the paper: an extensive notation table is given in \ref{A:notation}, \ref{A:summability} describes a post-processing technique for ensuring summability to unity conditions, \ref{A:sparsity} and \ref{A:postProcessing} present two approaches that can be utilized to enhance sparsity of solutions, and \ref{A:regularization} describes a means for selecting the regularization parameter.

\section{Derivation of Mixture Distribution from Traffic Flow Characteristics} \label{sec:TTdistDerivation}
\subsection{The Equilibrium Pace Function and its Properties}
Consider a vehicle traversing a path in a traffic network with terminal positions $x_1$ at the upstream end of the path and $x_2$ at the downstream end.  Let $C(x)$ denote the time instant that the vehicle crosses position $x$.  Then the travel time along the path is given by $C(x_2) - C(x_1)$. By definition, $C$ is continuous and strictly increasing.  Hence, the fundamental theorem of calculus provides a function $\Pi(x)$ such that
\begin{equation}
C(x_2) - C(x_1) = \int_{x_1}^{x_2} \Pi(x) \dd x, \label{E:tt1}
\end{equation}
where 
\begin{equation}
\Pi(x) = \frac{\dd C(x)}{\dd x}.
\end{equation}
The non-negative function $\Pi(x)$ is the \textit{pace} at position $x$ (in units of unit time per unit distance).  In a first-order macroscopic traffic flow context, the pace at position $x$ depends on the traffic density at $x$: let $\rho$ denote the traffic density, $V(\rho)$ an equilibrium speed-density relation, and $Q(\rho) = \rho V(\rho)$ an equilibrium flux function.  We denote the \textit{equilibrium pace} function by $P$, which is related to the equilibrium speed and flux functions via:
\begin{equation}
P(\rho) = \frac{\rho}{Q(\rho)} = \frac{1}{V(\rho)}.
\end{equation}
It follows that the equilibrium pace function has the following four properties:
\begin{itemize}
	\item[(i)] As $\rho \rightarrow 0$, $P(\rho) \rightarrow v_{\mathrm{fr}}^{-1}$, where $v_{\mathrm{fr}}$ is the free-flow speed.
	\item[(ii)] As $\rho \rightarrow \rho_{\mathrm{jam}}$, $P(\rho) \rightarrow \infty$, where $\rho_{\mathrm{jam}}$ is the jammed traffic density.
	\item[(iii)] $P(\rho)$ is continuous on $\rho \in (0,\rho_{\mathrm{jam}})$.
	\item[(iv)] $P(\rho)$ is non-decreasing in $\rho$.
\end{itemize}
Properties (i) - (iii) follow immediately from well-known properties of \textit{equilibrium} speed-density relations \citep{del1995functional}.  Property (iv) follows from 
\begin{equation}
\frac{\dd P(\rho)}{\dd \rho} = \frac{\dd}{\dd \rho} \Big(\frac{\rho}{Q(\rho)}\Big) = \frac{1}{Q(\rho)} \Big(1 - \frac{1}{V(\rho)} \frac{\dd Q(\rho)}{\dd \rho} \Big),
\end{equation}
and \textit{since traffic waves cannot move faster than the traffic itself}, we have for any $0 \le \rho \le \rho_{\mathrm{jam}}$ that $V(\rho) \ge \dd Q(\rho) / \dd \rho$.  Since $Q(\cdot)$ is non-negative, we immediately have that 
\begin{equation}
\frac{\dd P(\rho)}{\dd \rho} \ge 0
\end{equation}
for all $0 \le \rho \le \rho_{\mathrm{jam}}$ and, hence, $P(\cdot)$ is non-decreasing.  An example pace function, based on the Newell-Franklin speed-density relation \citep{newell1961nonlinear,franklin1961structure} is given in Figure \ref{f:Pace}. The speed relation is given by
\begin{align}
V(\rho) = v_{\mathrm{fr}} \Big(1 - e^{\frac{v_{\mathrm{b}}}{v_{\mathrm{fr}}} (\frac{\rho_{\mathrm{jam}}}{\rho} -1) }\Big), \label{E:N-F}
\end{align}
where $v_{\mathrm{b}}$ is the backward wave speed.
\begin{figure}[h!]
	\centering
	\resizebox{0.65\textwidth}{!}{%
		\includegraphics{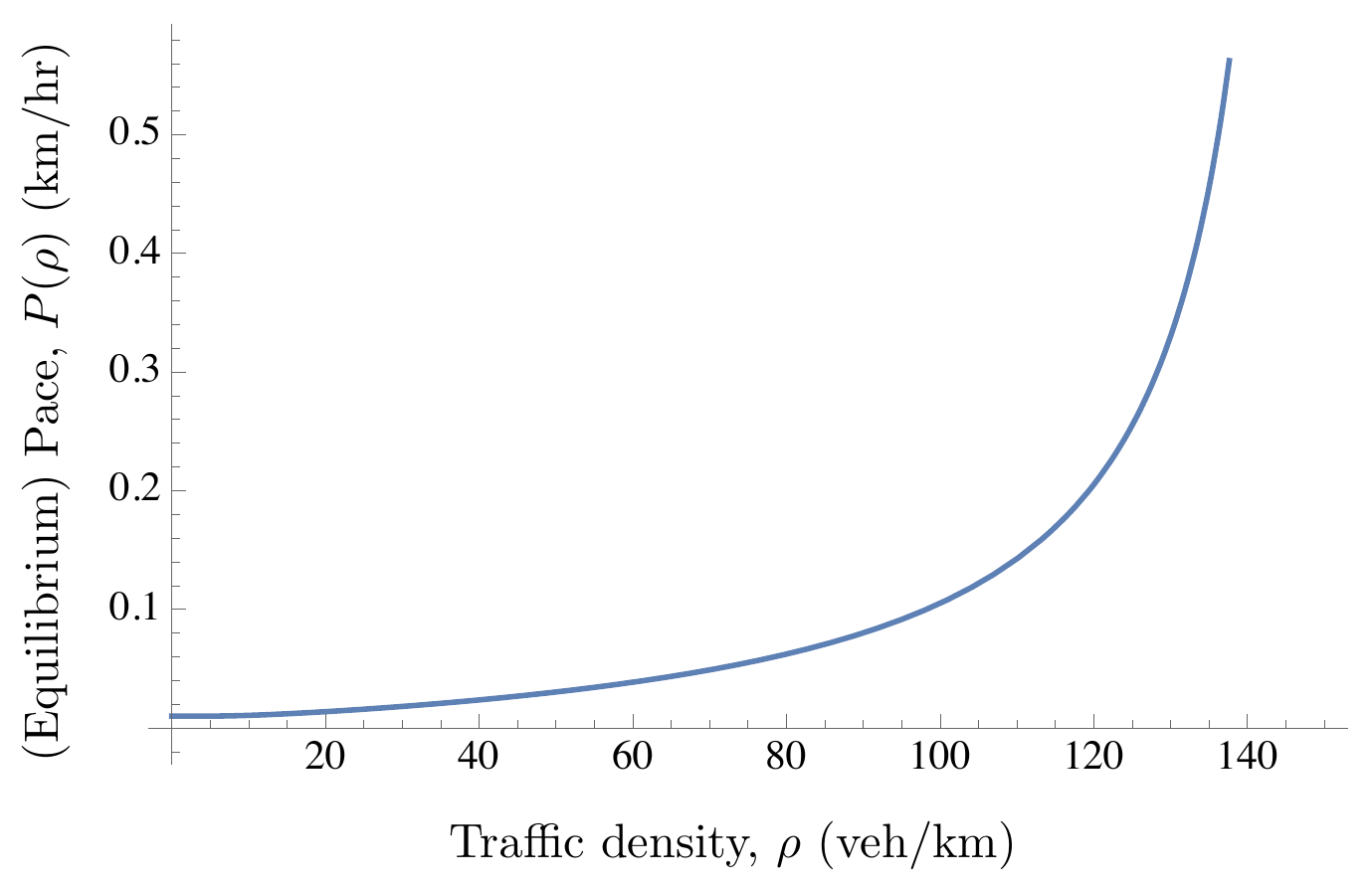}}
	\caption{An example pace function based on the Newell-Franklin speed-density relation with $v_{\mathrm{fr}} = 100$ km/hr, backward wave speed $v_{\mathrm{b}} = -20$ km/hr and $\rho_{\mathrm{jam}}=150$ veh/km.}
	\label{f:Pace}
\end{figure}

The properties above suggest that an appropriate choice for the distribution function of pace is one that is supported on $[v_{\mathrm{fr}}^{-1}, \infty)$.  A variety of known distribution functions are supported on positive intervals; among them the Gamma distribution \citep{polus1979study,kim2014finite,kim2015compound} and the Lognormal distribution \citep{richardson1978travel,rakha2006estimating,pu2011analytic,arezoumandi2011estimation} are most widely used as the distribution functions of traffic variables.  As a mixture density, the former offers some tractability properties that the latter does not.  

\subsection{The Distribution of Traffic Densities and Pace}
In \citep{carey2005convergence}, a discrete-space formulation of a few pace functions was considered and convergence to the \cite{lighthill1955kinematic} and \cite{richards1956shock} model (the LWR model) as the discrete space interval length approaches zero was demonstrated.  In a similar way, we define $\Pi(x) \equiv P\big( \rho(x), x \big)$, where we allow the equilibrium pace to depend on position, and treat dependence of density on time as implicit.  Hence \eqref{E:tt1} can be written as
\begin{equation}
C(x_2) - C(x_1) = \int_{x_1}^{x_2} P\big( \rho(x), x \big) \dd x. \label{E:tt2}
\end{equation}

Uncertainty about travel times can be interpreted as \textit{absence of (detailed) knowledge of traffic conditions}.  This is captured by treating $\rho(x)$ for each $x$ as a random variable, where $\PP(\rho(x) \le 0) = \PP(\rho(x) \ge \rho_{\mathrm{jam}}) = 0$ \textbf{must} hold.   This dictates distributions of traffic density that are supported on bounded intervals. \cite{haight1963mathematical} prescribes variants of the Beta distribution for traffic densities.  Beta distributions can be tuned to capture a variety of other distributions with bounded support as special cases.   For example, a Beta distribution with parameters $a = b = 1$ is a uniform distribution, which can be used to represent \textit{complete ignorance about traffic conditions}. Similarly, when the parameters are such that $a < b$, the distribution is positively skewed, which can be used to represent lower density traffic, while the case $a > b$ corresponds to negative skew, which represents high density traffic.  The probability density function (PDF) for traffic densities, given the parameters $a,b$ and $\rho_{\mathrm{jam}}$, can be written as
\begin{equation}
f_{\rho}(r) = \mathbb{1}_{\{ r \in [0,\rho_{\mathrm{jam}}] \}}\frac{1}{\rho_{\mathrm{jam}}^{a + b - 1} B(a,b)} r^{a - 1} \big(\rho_{\mathrm{jam}} - r \big)^{b - 1},
\end{equation}
where $B$ is the Beta function; this PDF is illustrated in Figure \ref{f:Kumaraswamy}.
\begin{figure}[h!]
	\centering
	\resizebox{0.65\textwidth}{!}{%
		\includegraphics{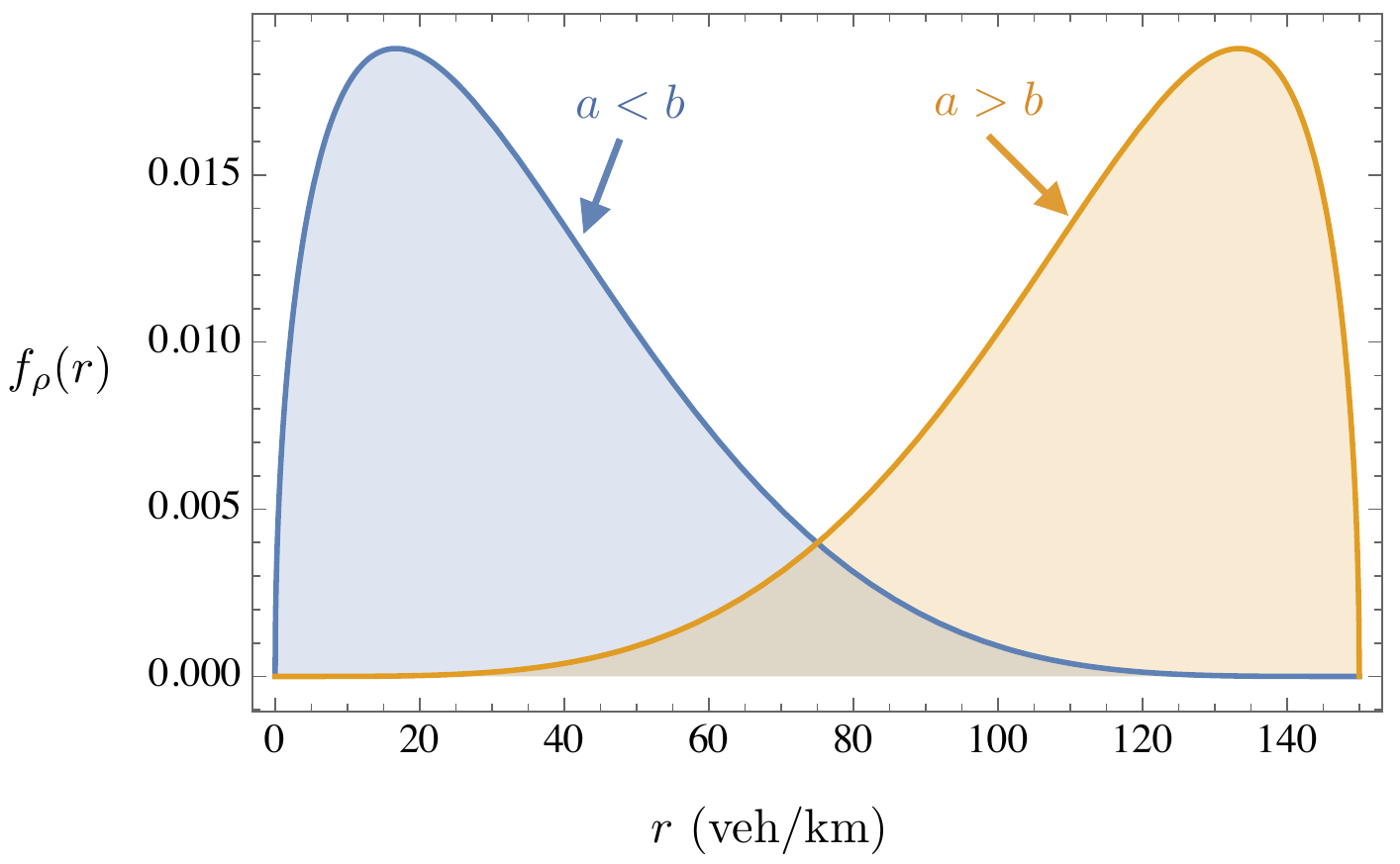}}
	\caption{Illustration of $f_{\rho}$ when $a < b$ (uncongested traffic) and $a > b$ (congested traffic).  In both cases, $\rho_{\mathrm{jam}} = 150$ veh/km.}
	\label{f:Kumaraswamy}
\end{figure}

In special cases, such as the Newell-Franklin relation \eqref{E:N-F}, which possesses a unique inverse and is differentiable, one obtains the distribution of pace directly from the pace function $P$ and the PDF of traffic density $f_{\rho}$. Denote the PDF of pace by $f_{\Pi}$; then
\begin{align}
f_{\Pi}(t)  = \mathbb{1}_{\{t v_{\mathrm{fr}}\ge 1\}} f_{\rho}\big( P^{-1}(t) \big) \frac{\dd P^{-1}(t)}{\dd t}.
\end{align}
In the general case, a unique inverse $P^{-1}$ does not exist (e.g., the widely used triangular relation).  In this case, the probability distribution of the pace at position $x$, $\Pi(x)$, can be derived from the properties of the equilibrium pace function $P$ and the distribution of traffic densities.  

\subsection{General Pace Functions and the Distribution of Travel Time} \label{ssec:TTdist}
To represent the distribution of travel times for general pace functions,  first note that continuity of equilibrium pace functions implies that there exists a polynomial that approximates any $P$ arbitrarily closely (by the Weierstrass approximation theorem).  Specifically, there exist weights $\{\zeta_k\}_{k \ge 0}$ such that, for any $\epsilon_1,\epsilon_2>0$,
\begin{align}
\Big| P(\rho) - \sum_{k=0}^{\infty} \zeta_k \rho^k \Big| \le \epsilon_1
\end{align}
for all $\rho \in [0, \rho_{\mathrm{jam}} - \epsilon_2]$.  In other words, on any closed interval that does not include $\rho_{\mathrm{jam}}$, there exists a polynomial that approximates $P$ uniformly (since $P$ has an asymptote at $\rho_{\mathrm{jam}}$).  Hence,
\begin{equation}
P(\rho) \approx \sum_{k=0}^{\infty} \zeta_k \rho^k \label{E:weierstrass}
\end{equation}
for all $\rho \in [0, \rho_{\mathrm{jam}})$. 
The characteristic function of $\Pi(x)$ is given by $\varphi_{\Pi}(s) \equiv \EE e^{is P(\rho)}$, where dependence on $x$ is made implicit, and $i \equiv \sqrt{-1}$ is the imaginary unit. Expanding $e^{is P(\rho)}$ using Maclaurin series, we have that
\begin{align}
\varphi_{\Pi}(s) &= \EE \sum_{j = 0}^{\infty} \frac{(is)^j}{j!} \big(P(\rho)\big)^j = 1 + \EE \sum_{j = 1}^{\infty} \frac{(is)^j}{j!} \big(P(\rho)\big)^j.
\end{align}
Expanding the terms involving the pace function using \eqref{E:weierstrass}, we have that
\begin{align}
\big(P(\rho)\big)^j &= \Big( \sum_{k=0}^{\infty} \zeta_k \rho^k \Big)^j = \sum_{k_1=0}^{\infty} \hdots \sum_{k_j=0}^{\infty} \prod_{l=1}^j \zeta_{k_l} \rho^{k_1 + \hdots + k_j}  = \sum_{m=0}^{\infty} \bigg( \sum_{k_1 + \hdots + k_j = m} \binom{m}{k_1 ~ \hdots ~ k_j} \prod_{l=1}^j \zeta_{k_l} \bigg) \rho^m.
\end{align}
Since the weights $\{\zeta_k\}_{k \ge 0}$ are constant (they depend on the pace function), for each $m$ the product term $\prod_{l=1}^j \zeta_{k_l}$ can be treated as a constant, which we denote by $C_m$. By the multinomial theorem, we have that $\sum_{k_1 + \hdots + k_j = m} \binom{m}{k_1 ~ \hdots ~ k_j} = j^m$.  Hence,
\begin{align}
\varphi_{\Pi}(s) = 1 +  \EE \sum_{j = 1}^{\infty} \frac{(is)^j}{j!} \sum_{m = 0}^{\infty} C_m j^m \rho^m = 1 +  \sum_{m = 0}^{\infty} \sum_{j = 1}^{\infty} \frac{(is)^j}{j!} C_m j^m \EE \rho^m.
\end{align}
For the PDF of traffic density \eqref{f:Kumaraswamy}, the $m$th moment is given by:
\begin{equation}
\EE \rho^m = \rho_{\mathrm{jam}}^m \frac{B(a + m, b)}{B(a,b)} = \rho_{\mathrm{jam}}^m\frac{\Gamma(a+b) \Gamma(a + m)}{\Gamma(a) \Gamma(a + b + m)}, 
\label{E:moments}
\end{equation}
where $\Gamma$ is the Gamma function. Define $\beta_m \equiv a + b + m$; then
\begin{align}
\varphi_{\Pi}(s) = 1 +  \sum_{m = 0}^{\infty} C_m  \rho_{\mathrm{jam}}^m \frac{\Gamma(a+b) \Gamma(a + m)}{\Gamma(a)} \sum_{j = 1}^{\infty} \frac{(is)^j}{j!} \frac{j^m}{\Gamma(\beta_m + j)} \frac{\Gamma(\beta_m + j)}{\Gamma(\beta_m)}.
\end{align}
By appeal to the second mean value theorem for integrals in conjunction with the integral test for convergence of infinite series, it can be shown that there exist constants $\widetilde{C}_m$ that depend on $m$ (but not $j$) such that
\begin{align}
\sum_{j = 1}^{\infty} \frac{(is)^j}{j!} \frac{j^m}{\Gamma(\beta_m + j)} \frac{\Gamma(\beta_m + j)}{\Gamma(\beta_m)} = \widetilde{C}_m \sum_{j = 1}^{\infty} \frac{(is)^j}{j!} \frac{\Gamma(\beta_m + j)}{\Gamma(\beta_m)}.
\end{align} 
Defining the weights $\theta_m \equiv C_m  \rho_{\mathrm{jam}}^m \frac{\Gamma(a+b) \Gamma(a + m) \widetilde{C}_m }{\Gamma(a)}$, we have that 
\begin{align}
\varphi_{\Pi}(s) = 1 +  \sum_{m = 0}^{\infty} \theta_m \sum_{j = 1}^{\infty} \frac{(is)^j}{j!} \frac{\Gamma(\beta_m + j)}{\Gamma(\beta_m)}.
\end{align}
Note that (i) the characteristic function of a Gamma PDF with shape parameter $\beta$ and scale parameter $\sigma$, denoted $\varphi_{\Gamma}(s)$, is given by
\begin{equation}
\varphi_{\Gamma}(s) = \big(1 - \sigma is\big)^{-\beta} = \sum_{j=0}^{\infty} \frac{(is)^j}{j!} \sigma^j \frac{\Gamma(\beta + j)}{\Gamma(\beta)} \label{E:CharGamm}
\end{equation}
and that (ii) the characteristic function associated with a mixture distribution is a mixture of the characteristic functions of the component distributions.  That is, if a PDF is given by
\begin{align}
f_X(x) = \sum_{m=0}^{\infty} \theta_m f_m(x)
\end{align}
for some random variable $X$ with range $\mathcal{R} \subseteq \RR$, where $\{f_m\}_{m \ge 0}$ are mixture component PDFs, then
\begin{align}
\varphi_X(s) & = \int_{\mathcal{R}} e^{isx} f_X(x) \dd x = \int_{\mathcal{R}} e^{isx}  \sum_{m=0}^{\infty} \theta_m f_m(x) \dd x = \sum_{m=0}^{\infty} \theta_m \int_{\mathcal{R}} e^{isx} f_m(x) \dd x = \sum_{m=0}^{\infty} \theta_m \varphi_m(s),
\end{align}
where $\{\varphi_m\}_{m \ge 0}$ are the characteristic functions associated with the component PDFs.  These two properties imply that the PDF of pace $f_{\Pi}$ can be captured by an appropriately tuned mixture of Gamma PDFs with scale parameter $\sigma = 1$ and shape parameters $\{\beta_m\}_{m \ge 0}$ provided that $\sum_{m=0}^{\infty} \theta_m = 1$.  When the latter holds, we have that
\begin{align}
\varphi_{\Pi}(s) = \sum_{m=0}^{\infty} \theta_m \sum_{j = 0}^{\infty} \frac{(is)^j}{j!} \frac{\Gamma(\beta_m + j)}{\Gamma(\beta_m)} = \sum_{m=0}^{\infty} \theta_m (1 - is)^{-\beta_m}.
\end{align}
The PDF of a Gamma distributed random variable with shape parameter $\beta$ and scale parameter $\sigma$ is given by
\begin{equation}
f_{\Gamma}(t;\beta,\sigma) = \mathbb{1}_{\{t\ge 0\}} \frac{1}{\sigma \Gamma(\beta)} \Big(\frac{t}{\sigma}\Big)^{\beta - 1} e^{-\frac{t}{\sigma}}. \label{eq_gammaDen}
\end{equation}
Hence, the distribution of pace can be written as 
\begin{align}
f_{\Pi}(t) = \sum_{m=0}^{\infty} \theta_m f_{\Gamma}(t;\beta_m,1).
\end{align}

The result above generalizes immediately from pace to travel time: it may be assumed that $P$ is continuous in $x$ since lane additions/drops do not occur abruptly and when speed limits change, drivers cannot adjust their speeds instantaneously, since equilibrium relations are governed by driving behavior -- see \citep{jabari2014probabilistic,jabari2018stochastic,zheng2018traffic}.  Thus by the mean-value theorem, there exists $x_1 \le \overline{x} \le x_2$ such that
\begin{equation}
P\big( \rho(\overline{x}), \overline{x} \big)(x_2 - x_1) = \int_{x_1}^{x_2} P \big( \rho(x), x \big) \dd x.
\end{equation}
Hence, the travel time along the path starting at $x_1$ and terminating at $x_2$ can be represented by the pace evaluated at an ``intermediate location''.  That is, $C(x_2) - C(x_1) = P\big( \rho(\overline{x}), \overline{x} \big)(x_2 - x_1)$ and the same procedure applied to represent the distribution of pace as a mixture can be applied to travel time.  Specifically, by continuity we have that
\begin{align}
C(x_2) - C(x_1) = P\big( \rho(\overline{x}), \overline{x} \big)(x_2 - x_1) = \sum_{k=0}^{\infty} \widetilde{\zeta}_k \rho^k,
\end{align}
where we write $\{\widetilde{\zeta}_k\}_{k \ge 0}$ to distinguish the weights associated with travel time from those associated with pace.  Following the same procedure above, we can obtain
\begin{align}
f_{T}(t) = \sum_{m=0}^{\infty} \theta_m f_{\Gamma}(t;\beta_m,1), \label{E:gammaMix}
\end{align}
where $f_T$ is the PDF of travel time.  We close this section with some remarks about the  mixture distribution \eqref{E:gammaMix}.
\begin{enumerate}
	\item The most commonly used mixture densities (e.g., Gaussian, biweight, and Epanechnikov) all suffer from assigning non-zero probability to negative travel times. Gamma mixture densities overcome this drawback.
	\item In the derivation above, the shape parameters $\{\beta_m\}_{m \ge 0}$ are arbitrary; the specifics of the distribution of the equilibrium pace function $P$ being subsumed into the mixture weights, $\{\theta_m\}_{m \ge 0}$.  The shape parameters, therefore, can be chosen \emph{a priori}.
	\item Since $\sigma = 1$ for all $m$, we have that the shape parameters bear the sole responsibility of determining the locations of the mixture components.  The locations are represented by the peaks of the distributions, located at the modes, which are given by $\{\beta_m -1\}_{m \ge 0}$.
	\item There are three main drawbacks of the mixture distribution above:
	\begin{enumerate}
		\item The shape of each of the component distributions depends on location: the variance of component distribution $m$ is given by $\beta_m$.  This results in an undesirable feature referred to as boundary bias. We address this in \autoref{ssec:boundarybias}.
		\item We lose some flexibility (and model parsimony) as a result of fixing the scale parameters to a single value $\sigma=1$.  We address this issue in \autoref{ssec:ml}, where we propose a generalization of the component PDFs that allow for variable scale parameters.
		\item The mixture involves an infinite number of components, which renders it infeasible from an estimation standpoint.  In the following sections, we set the mixture to have $M < \infty$ components, where $M$ is is chosen to be sufficiently large.  We address the errors associated with truncation in \autoref{lemNonAdaptive} and \autoref{lemML} and present a zero-overhead post-processing step in \ref{A:summability} to ensure that $\sum_{m=0}^M \theta_m = 1$ is satisfied.
	\end{enumerate}
\end{enumerate}

\section{Empirical Travel Time Distribution and Sparse Estimation}
\label{sec:sparseDenEst}
This section presents the estimation problem that we seek to solve.  In essence, we seek to find a mixture distribution that most closely resembles the distribution of the travel time data.  For the latter, we propose the use of a generalization of a histogram in which the histogram bins can are replaced by \textit{kernels}, which can be represented by any PDF.  The rectangular bins of a typical histogram can be seen as a special case of this, where the chosen kernel is a uniform PDF.

\subsection{Parzen Density Estimator: Empirical Distribution}
\label{ssec:denEstProb}
Given $S$ samples $T_1,\hdots,T_S$ drawn from a population with (unknown) probability density function  $f$, the \textit{Parzen density}, also known as \emph{Parzen window (PW) estimator} \citep{parzen1962estimation,cacoullos1966estimation,raudys1991effectiveness,silverman1986density} of travel time $t$ is given by:
\begin{equation}\label{eq:Parzen}
	\widehat{f}(t) = \frac{1}{S} \sum_{j=1}^S \kappa_h(t-T_j),
\end{equation}
where $\kappa_h$ is a window (or kernel) of width $h$, and $h$ is called the smoothing parameter. The Parzen density can equivalently be interpreted as a modified  histogram, allowing for the ``bins'' to be non-rectangular. 
As an example, choosing $\kappa_h \equiv \delta$, where $\delta$ is the Dirac delta function, we get the standard empirical distribution $\frac{1}{S} \sum_{j=1}^S \delta(t-T_j)$,  which uses kernels with zero bandwidth, $h=0$. Typically, $\kappa_h$ is a PDF; for example, we use the Gaussian density with variance $h^2$ in our experiments. 
Several methods have been proposed to determine $h$ based either on minimizing the mean square error or based on cross-validation techniques (see \citep{lacour2016estimator} and references therein for a contemporary treatment of the bandwidth selection problem). 

Parzen window (PW) estimators can also be regarded as a special type of finite mixture models, where the mixture components are assigned equal weights and are located exactly at the training data. The PW estimator generally requires as many components as the number of training samples.  As a result, it may require substantial storage requirements.  In this paper, the PW estimators serve as empirical distribution functions (or generalized histograms), and the goal is to develop and fit parsimonious (light-weight) models that may as well achieve higher (out-of-sample) prediction accuracy.

\subsection{Sparse Mixture Density Estimation} \label{ssec:estimationProblem}
Consider the mixture density
\begin{equation}
	\overline{f}(t) = \sum_{m=0}^{M-1} \theta_m \phi_m(t), \label{kernel_pdf}
\end{equation}
where $\{\phi_m\}_{m=0}^{M-1}$ are the component density functions ($M$ in total) and $\{\theta_m\}_{m=0}^{M-1}$ are the component weights.  
We will allow $M$ to be large so that \eqref{kernel_pdf} is rich enough to fit a broad class of distributions.  Our aim is to achieve a sparse representation of $\overline{f}$ (a parsimonious fit), i.e., one with most of the elements of the vector $\theta = [\theta_0 \hdots \theta_{M-1}]^{\top}$ being zero  while maintaining test performance or generalization capability comparable to that of the PW estimate obtained with an optimized bandwidth $h$.  We thus seek to solve:
\begin{equation}
	\underset{\theta \in \Omega}{\mathrm{minimize}} ~~ \frac{1}{2} \Big\lVert \widehat{f} - \sum_{m=0}^{M-1} \theta_m \phi_m \Big\rVert_2^2 + w \lVert \theta \rVert_1, \label{eq_sparseDenProblem}
\end{equation}
where $\Omega \subseteq \RR^M$ is a set of $M$-dimensional vectors that we consider for the optimization problem, and $w\ge0$ is the regularizing parameter. The $L_2-$norm is over a suitably chosen (infinite-dimensional) functional space and the $\ell_1-$norm is the usual (finite dimensional) vector norm, i.e., the sum of absolute values of the vector entries. 
Typically, 
\begin{equation}
	\frac{1}{2} \big\lVert \widehat{f} - \sum_{m=0}^{M-1} \theta_m \phi_m \big\rVert_2^2 = \frac{1}{2} \int_{\RR_{+}} \Big( \widehat{f}(t) - \sum_{m=0}^{M-1} \theta_m \phi_m(t) \Big)^2 \mathrm{d}t.
\end{equation}
The first term in the objective function, $\| \widehat{f} - \sum_{m=0}^{M-1} \theta_m \phi_m\|_2^2$, is a measure of \textit{goodness-of-fit}: it is the (squared) $L_2-$distance between the empirical distribution (of the data) $\widehat{f}$ and the fitted distribution $\overline{f}$.  The second term is an $\ell_1-$\emph{regularizer}: $\|\cdot\|_1$ is known to promote sparsity in the vector of weights $\theta$ \citep{tibshirani1996regression},  i.e., a parsimonious solution. Finally, note that a higher value for $w$ yields higher  sparsity of the optimal solution vector $\theta$ of the optimization problem~\eqref{eq_sparseDenProblem}. 

\subsection{Support Discretization}
\label{sec:discrete}
To solve the estimation problem \eqref{eq_sparseDenProblem} via numerical optimization, we discretize the travel times. This is done by defining disjoint intervals (of the same or variable lengths) in the support of the distribution and associate with each interval a \textit{representative} value (denoted by $\tau_n$ for the $n$th interval, e.g., its midpoint). Each data point is assigned the representative value of the interval it lies in: 
let $\tau$ be a surjective mapping from the continuous interval $[0 , T_{\max}]$ into the discrete set $\{\tau_n\}_{n=0}^{N-1}$, i.e., $\tau$ performs the operation $t \mapsto \tau_n$. In effect, the function $\tau$ takes a continuous travel time $t$ and returns its representative $\tau_n$ in the discrete set.  Consequently, the PDFs $\widehat{f}$ and $\overline{f}$ are approximated by vectors of size $N$, denoted, respectively, by $\widehat{p}$ and $\overline{p}$. 
We thus have for any $t \ge 0$ that
\begin{equation}
	\overline{f}(t) \approx \overline{p}_{\tau(t)} = \sum_{m=0}^{M-1} \theta_m \phi_m(\tau(t)).
\end{equation}

The locations of the $M$ component densities simply constitute a set of travel times, which we denote by $\{t_m\}_{m=0}^{M-1}$; note that these do not necessarily coincide with the discrete support of the distribution. Besides, we will consider mixture components with variable width, so that $\{t_m\}_{m=0}^{M-1}$ may not have $M$ distinct values, i.e., some values coincide (this corresponds to the case of placing multiple mixture components of different width at the same location), and $M>N$ is possible. The \emph{distinct values} in $\{t_m\}_{m=0}^{M-1}$ are taken to be a subset of the discrete support of the distribution $\cup_{m=0}^{M-1} \{t_m\} \subseteq \{\tau_n\}_{n=0}^{N-1}$; denoting the number of distinct values in the set $\{t_m\}_{m=0}^{M-1}$ by $M'$, we have necessarily that $M'\le N$. 
When a single scale parameter is used (see \autoref{sec:kernelChoice}) it holds that $M=M'\le N$.

The $M$ mixture components are further quantized in accordance with the discretization of the support as follows: we define $\phi_{n,m} \equiv c_{n,m} \phi_m(\tau_n)$, where $c_{n,m}$ is a constant that depends on the discretization method, and the $m$-th mixture component function. Similarly, we may quantize the PW by setting $\widehat{p}_n = \alpha_n \widehat{f}(\tau_n)$. In essence, $\alpha_n$ is a measure of the width of the $n$th interval; for example $\alpha_n \equiv \Delta$ for a uniform discretization with step-size $\Delta$ (e.g., we use a  second-by-second uniform discretization in our experiments and set $\alpha_n \equiv 1$). We discuss the issue of mixture component discretization (the selection of $\{c_{n,m}\}$) in detail in \autoref{ssec:discretization} and \autoref{ssec:ml}. Finally, defining the  matrix $\Phi \equiv [\phi_{n,m}] \in \RR_+^{N \times M}$,  
we have
\begin{equation}
	\overline{p} = \Phi \theta
\end{equation}
and we consider, in the following, the (discrete) estimation problem:
\begin{equation}
\underset{\theta \in \Omega}{\mathrm{minimize}} ~~ \frac{1}{2} \big\lVert \widehat{p} - \Phi \theta \big\rVert_2^2 + w \lVert \theta \rVert_1 = \frac{1}{2} \sum_{n=0}^{N-1} \Big( \widehat{p}_n - \sum_{m=0}^{M-1} \phi_{n,m} \theta_m \Big)^2 + w \sum_{m=0}^{M-1} |\theta_m|, \label{eq_LASSO}
\end{equation}
which is known as the (constrained) \emph{Least Absolute Shrinkage and Selection Operator (LASSO)} in the statistics and machine learning literature \citep{tibshirani1996regression}. 

\section{Gamma and Mittag-Leffler Mixtures}
\label{sec:kernelChoice}
In this section, we specialize the estimation problem presented in \autoref{ssec:estimationProblem} to the Gamma density mixture derived in \autoref{sec:TTdistDerivation} and address the drawbacks presented at the end of \autoref{ssec:TTdist}.

\subsection{Boundary Bias}
\label{ssec:boundarybias}
Like most asymmetric densities, the shape of the Gamma density (specifically, its width) depends on both the scale parameter \textit{as well as} the location parameter (e.g., its mean). This change in shape results in what is referred to as \textit{boundary bias} in the statistics literature \citep{chen2000probability} and is addressed by changing the roles of parameter and argument. This is done as follows: to evaluate the probability density at $t$, the model (i) uses a single Gamma PDF with its mode, equal to $(\beta-1)\sigma$, coinciding with $t$, (ii) evaluates the densities of the sample points using this function, and (iii) calculates a weighted sum of these densities.  Effectively, \textit{the roles of parameter and argument are reversed}.   
This prevents bias from mixture densities located near the boundaries.  
Specifically, placement of the mixture densities is done as follows: 
In order to locate the (mode  of the) mixture component at the argument $t$, the location parameter is set so that $(\beta-1)\sigma=t$, i.e., we set $\beta = 1 + \frac{t}{\sigma}$;  hence, for a given scale parameter $\sigma>0$, the $m$-th Gamma density is given by:
\begin{equation}\label{eq:Gamma_kernel_def}
	\phi_m(t) = f_{\Gamma} \Big(t_m; 1+\frac{t}{\sigma}, \sigma \Big).
\end{equation}
The estimated probability (before discretization) is then given by: 
\begin{equation}
	\overline{f}(t) = \sum_{m=0}^{M-1} \theta_m f_{\Gamma} \Big(t_m; 1+\frac{t}{\sigma}, \sigma \Big) = \sum_{m=0}^{M-1} \theta_m \frac{1}{\sigma \Gamma\big( 1 + \frac{t}{\sigma} \big)} \Big(\frac{t_m}{\sigma} \Big)^{\frac{t}{\sigma}} e^{-\frac{t_m}{\sigma}}.
\end{equation}
This mechanism is illustrated in Figure \ref{f2_R1}. 
\begin{figure}[h!]
	\centering
	\resizebox{0.75\textwidth}{!}{%
		\includegraphics{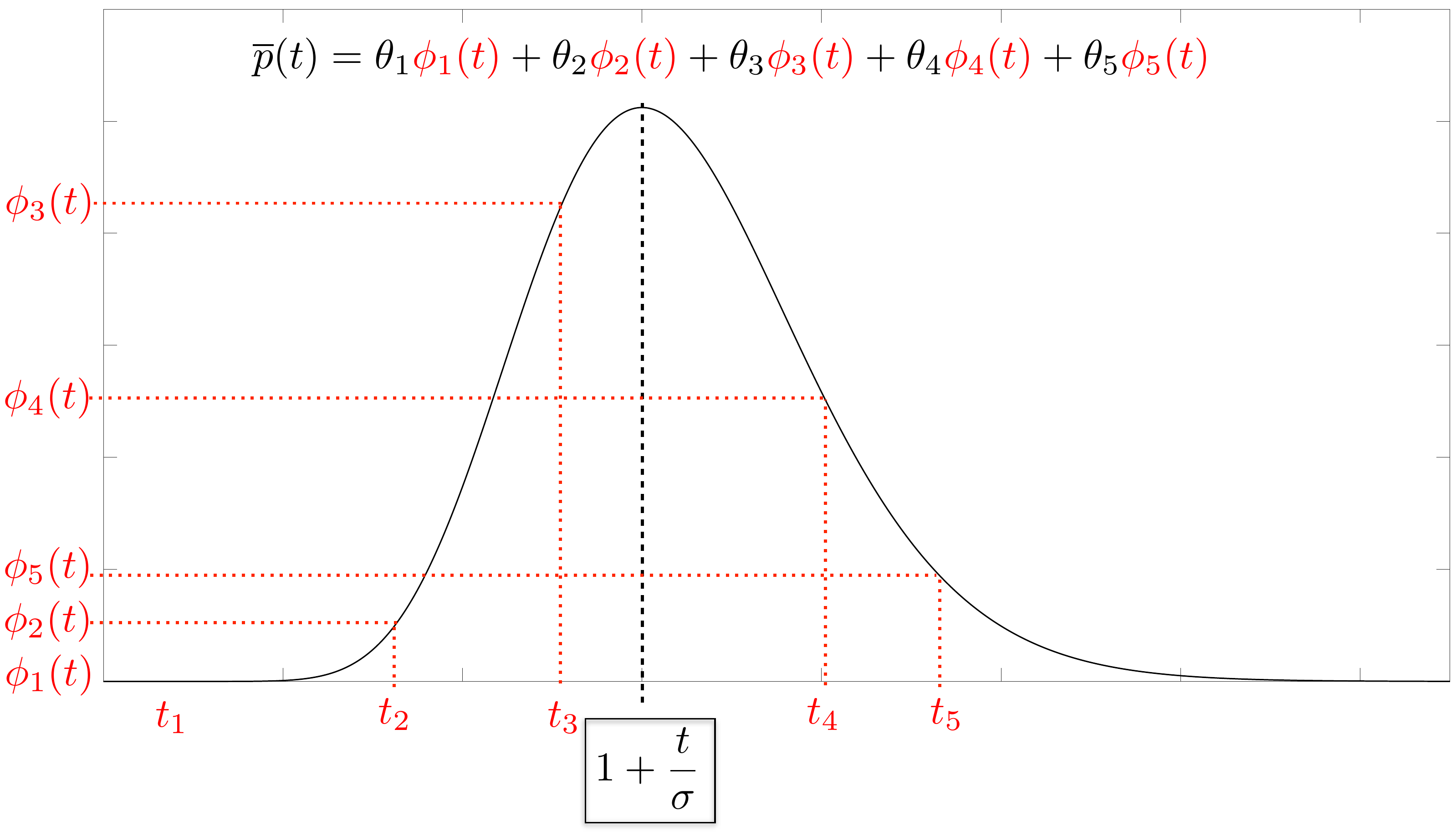}}
	\caption{Mixture of Gamma mixture densities: Gamma probability density function centered at $t$ and evaluated at times $\{t_m\}_{m=1}^5$ with weights $\{\theta_m\}_{m=1}^5$.}
	\label{f2_R1}
\end{figure}

\subsection{Discretization of Gamma Mixture Densities} 
\label{ssec:discretization}
Gamma mixture densities in the statistics literature do not in general integrate to unity.  In other words $\int_0^{\infty} f_{\Gamma}(t_m; 1+\sigma^{-1}t,\sigma) \dd t \ne 1$: this is a consequence of reversing the roles of parameter and argument.  In this paper, we give close attention to this issue  and ensure that our approach guarantees that $\overline{p}$ is a valid probability mass function (PMF).  A similar analysis can be carried out for discretizing $\widehat{p}$.

The standard approach of normalizing $\overline{p}$ as a post-processing step is not applicable in our case.  In standard practice, one only considers goodness-of-fit: given that $\widehat{p}$ sums to unity, it follows that minimizing the distance $\| \widehat{p} - \overline{p} \|_2^2$ should yield a $\overline{p}$ that sums  close enough to unity, so that normalizing by $\sum_{n=0}^{N-1} \overline{p}_n$ will not incur a significant impact on the goodness-of-fit.  In our case, however, we have a trade-off between goodness-of-fit and sparsity (parsimony).  This type of normalization can \emph{substantially affect} the goodness-of-fit for a given sparsity level.  

Therefore, we carefully design the discretization in such a way that the resulting densities sum to unity (approximately). Mathematically, we require that $\sum_{n=0}^{N-1} \overline{p}_n \approx 1$, where the approximation error is kept below a predefined threshold.  Since each vector $\{\phi_{n,m}\}_{n=0}^{N-1}$ is interpreted as a probability distribution, we will first require that $\sum_{n=0}^{N-1} \phi_{n,m} \approx 1$ for all $m\in \{0,\hdots,M-1\}$ (equivalently, we require that $\Phi$ is approximately column-stochastic). We propose a choice for the set of discretization constants, denoted $\{c_n\}_{n=0}^{N-1}$, so that this is indeed the case.

\begin{proposition}[Non-Adaptive Kernel Density]\label{lemNonAdaptive}
	For $\Delta > 0$, define $\tau_n \equiv n\Delta$ and $t_m \equiv m\Delta$, where $n \in \{0,1,\hdots\}$ and $\{m=0,\hdots,M-1\}$.  Define $\widetilde{\Delta} \equiv \frac{\Delta}{\sigma}$ and $\phi_{n,m} \equiv c_n f_{\Gamma} \big(t_m; 1+ n\widetilde{\Delta}, \sigma \big)$ and set $c_n \equiv \sigma$ for all $n$.  Then, there exists $N < \infty$ such that, for any $0 < \varepsilon < 1$, 
	\begin{equation}
	1-\varepsilon \le \sum_{n=0}^{N-1} \phi_{n,m} \le 1.
	\end{equation}
\end{proposition}
\begin{proof}
Set the scale parameter so that $\widetilde{\Delta} = 1$, i.e., $\sigma \equiv \Delta$, then
\begin{equation}
	\sum_{n=0}^{N-1} \phi_{n,m} =  \sum_{n=0}^{N-1} c_n f_{\Gamma} \big(t_m; 1+ n\widetilde{\Delta}, \sigma \big) = \sum_{n=0}^{N-1} \frac{1}{\Gamma(1 + n\widetilde{\Delta})} \Big(\frac{t_m}{\sigma}\Big)^{n\widetilde{\Delta}} e^{-\frac{t_m}{\sigma}} \label{eq_discreteGammaKernel}
\end{equation}
converges to 1 (from below) as $N \rightarrow \infty$ since the terms inside the sum on the right hand side of \eqref{eq_discreteGammaKernel} take the form of the probability mass function of a Poisson distributed random variable with rate parameter $\frac{t_m}{\sigma}$.  We choose $N$ so that the sum is approximately unity: let $X$ be a Poisson random variable with rate parameter 
\begin{equation}
\underset{\{ t_m \}_{m=0}^{M-1}}{\max} ~ \frac{t_m}{\sigma} = \frac{(M-1)\Delta}{\sigma} =  M-1
\end{equation} 
and choose $N$ as the $(1-\varepsilon)$-percentile point of $X$.  That is
\begin{equation}\label{eq:tailGamma}
N \equiv \min \big\{n:  \PP(X \ge n) \le \varepsilon \big\}.
\end{equation}
This completes the proof. 
\end{proof}

Note that (i) choosing the rate parameter as $\max_{m=0,\hdots,M-1} ~ \frac{t_m}{\sigma}$ renders our choice of $N$ independent of $t_m$ and ensures that the threshold error is not exceeded for any $t_m$; (ii) this is only achievable when $N > M$; and (iii) ensuring that \emph{exactly} $\sum_{n=0}^{N-1} \phi_{n,m} = 1$ for all $j \in \{0,\hdots,M-1\}$ (as opposed to it being approximately equal to unity) can be achieved by re-defining $\phi_{0,m}$ (a constant which depends on $m$) as

\begin{equation}\label{eq:phi0}
\phi_{0,m} \equiv (\varepsilon_m+1) e^{-\frac{t_m}{\sigma}},
\end{equation}
where
\begin{equation}\label{error_eq}
\varepsilon_m = \sum_{n=N}^{\infty} \frac{1}{n!} \Big(\frac{t_m}{\sigma} \Big)^n.
\end{equation}
Note that $\varepsilon_m$ is increasing in $m$; therefore, it is upper-bounded by $\varepsilon_{M-1}$.  

\subsection{Adaptive Model using Mittag-Leffler Functions}\label{ssec:ml}
One drawback of the approach outlined above is the necessity for a single scale parameter $\sigma$.  To allow for varying scale parameters, we generalize the Gamma densities using \emph{Mittag-Leffler} functions.  In this context, the assumption that $\widetilde{\Delta} = 1$ is no longer feasible since $\sigma$ is allowed to vary from one mixture component to another.  To ensure summability to unity, we generalize the Gamma density to one which uses a generalized form of the exponential function.  This can be achieved by replacing $e^{-\frac{t}{\sigma}}$ in \eqref{eq_gammaDen} with the reciprocal of the (scaled) Mittag-Leffler function  \citep{haubold2011mittag}: 
\begin{equation}
	E_{\nu}(t) \equiv \sum_{n=0}^{\infty} \frac{t^n}{\Gamma(1 + n\nu)}.
\end{equation}
Note that the exponential function is a special case of the Mittag-Leffler function obtained when $\nu=1$; i.e., $E_1(t) \equiv e^t$.  
We first generalize \eqref{eq_gammaDen} as follows:
\begin{equation}
	f_{\mathrm{M-L}}(t;\beta,\sigma,\nu) = \frac{1}{\sigma \Gamma(\beta)} \Big(\frac{t}{\sigma}\Big)^{\beta - 1}  \Big[ E_{\nu} \Big( \big(\frac{t}{\sigma}\big)^{\nu} \Big) \Big]^{-1}, \label{eq_MLDen}
\end{equation}
where the parameters $\beta$ and $\sigma$ are the location and scale parameters defined above and the parameter $\nu$ depends on the discretization.  Proposition \ref{lemML} generalizes the summability result in Proposition \ref{lemNonAdaptive} to the adaptive case (i.e., varying scale parameters).  It also proposes a choice for the discretization constants, $\{c_{n,m}\}_{n,m=0}^{\infty,M-1}$, where the discretization varies by mixture component.
\begin{proposition}[Mittag-Leffler Densities]\label{lemML}
	Let $\Delta$ and $\{\tau_n\}_{n=0}^{\infty}$ be as defined in Proposition \ref{lemNonAdaptive}.  Assume the pairs $\{t_m,\sigma_m\}_{m=0}^{M-1}$ are sorted in increasing order and set $c_{n,m} \equiv \sigma_m$ for all $n,m$ pairs. Define $\widetilde{\Delta}_m \equiv \frac{\Delta}{\sigma_m}$ and the Mittag-Leffler densities
	\begin{align}
	\phi_{n,m} \equiv c_{n,m} f_{\mathrm{M-L}}(t_m;1+n\widetilde{\Delta}_m,\sigma_m,\widetilde{\Delta}_m) = \frac{1}{\Gamma(1+n \widetilde{\Delta}_m)} \Big(\frac{t_m}{\sigma_m}\Big)^{n \widetilde{\Delta}_m} \Big[E_{\widetilde{\Delta}_m}\Big(\big(\frac{t_m}{ \sigma_m}\big)^{\widetilde{\Delta}_m}\Big)\Big]^{-1} \label{eq_genGammaDen}
	\end{align}
	for $n = 0, \hdots$ and $m = 0, \hdots, M-1$. Then, there exists $N < \infty$ such that, for any $0 < \varepsilon < 1$, 
	\begin{equation}
	1-\varepsilon \le \sum_{n=0}^{N-1} \phi_{n,m} \le 1.
	\end{equation}
\end{proposition}
\begin{proof}
For each $0 \le m \le M-1$, let $\widetilde{X}_m$ be the generalized hyper-Poisson random variable proposed by \citep{chakraborty2017mittag} and let $p_{\widetilde{X}_m}$ denote its probability mass function with parameters $a_m$ and $b_m$:
\begin{equation}
	\PP(\widetilde{X}_m = n) \equiv p_{\widetilde{X}_m}(n; a_m,b_m) = a_m^n \frac{1}{ \Gamma(1 + nb_m) E_{b_m}(a_m)}. \label{eq_hyperPoisson}
\end{equation}
Set $a_m \equiv \big(t_m/\sigma_m\big)^{\widetilde{\Delta}_m}$ and $b_m \equiv \widetilde{\Delta}_m=\Delta / \sigma_m$.  Then for each $m$, the set $\{\phi_{n,m}\}_{n=0}^{\infty}$ in \eqref{eq_genGammaDen} is a probability mass function of a hyper-Poisson random variable.  Hence, 
\begin{equation}
\underset{N \uparrow \infty}{\lim} \sum_{n=0}^{N-1} \phi_{n,m} = 1.
\end{equation}
Setting 
\begin{align}
N \equiv \min \big\{n:  \PP(\widetilde{X}_m \ge n) \le \varepsilon, \mbox{ for } 0 \le m \le M-1 \big\} 
= \min \big\{n:  \PP(\widetilde{X}_{M-1} \ge n) \le \varepsilon \big\}
\end{align}
completes the proof. 
\end{proof}

We may ensure that $\sum_{n=0}^{N-1} \phi_{n,m} = 1$ \emph{exactly} by re-defining 
\begin{equation}
\phi_{0,m} \equiv (\widetilde{\varepsilon}_m+1) \big[E_{\widetilde{\Delta}_m}\big((t_m / \sigma_m)^{\widetilde{\Delta}_m}\big)\big]^{-1} ,
\end{equation} 
where	
\begin{equation}
\widetilde{\varepsilon}_m = \sum_{n=N}^{\infty} \frac{1}{\Gamma(1+n\widetilde{\Delta}_m)} \Big(\frac{t_m}{\sigma_m} \Big)^{n\widetilde{\Delta}_m}. \label{eq_errorFix}
\end{equation}
Observe that in a Mittag-Leffler (M-L) mixture, multiple mixture densities (of variable scale $\sigma_m$) may be associated with the same travel time $t_m$ (same location parameter); this implies that $M>N$ is possible. Similarly, the set $\{\sigma_m\}_{m=0}^{M-1}$ need not have distinct values.  Nonetheless, the above analysis shows that necessarily $N>M'$, where $M'$ denotes the number of distinct values of the times $\{t_m\}_{m=0}^{M-1}$. 

\section{Numerical Optimization}
\label{sec:numerics}
In what follows, the (constrained) LASSO problem~\eqref{eq_LASSO} is considered by taking $\Omega = \RR^{M}_+$ (the positive orthant) as opposed to $\Omega=\{\theta\in \RR^{M}_+ | \sum_{m=0}^{M-1} \theta_m = 1\}$
(the probability simplex). This is done purposefully for two reasons. First, this choice yields more efficient numerical optimization methods, which is especially important for real-time learning (effectively, the projection to the positive orthant is much simpler than the projection to the simplex, which requires sorting). Second, and more importantly, setting $\Omega=\{\theta\in \RR^{M}_+ | \sum_{m=0}^{M-1} \theta_m = 1\}$ would result in the optimization problem:

\begin{eqnarray*}\label{eq:consLASSOorthant}
	\underset{\theta \ge 0}{\mathrm{minimize}} & \frac{1}{2} \big\lVert \widehat{p} - \Phi \theta \big\rVert_2^2 + w \mathbf{1}^\top\theta\\
	\textrm{subject to} & \mathbf{1}^\top\theta = 1,\nonumber 
\end{eqnarray*}
which is equivalent to 
\begin{eqnarray*}\label{eq:consLASSOorthant1}
	\underset{\theta \ge 0}{\mathrm{minimize}} & \frac{1}{2} \big\lVert \widehat{p} - \Phi \theta \big\rVert_2^2\\
	\textrm{subject to} & \mathbf{1}^\top\theta = 1 \nonumber 
\end{eqnarray*}
since the second term in the objective is determined by the equality constraint ($\mathbf{1}$ is a vector of 1s of size $M$). 
This \emph{leaves no control over sparsity}, since the objective no longer depends on the control parameter $w$. This is clearly an undesirable feature when aiming for parsimonious solutions in a controllable fashion, and justifies our choice of selecting $\Omega = \RR_+^M$ in what follows. 
Ensuring that $\sum_{n=0}^{N-1} \overline{p}_n = 1$ (exactly) can be achieved with a zero-overhead post-processing mechanism.  This strategy is described in \ref{A:summability}.

LASSO~\eqref{eq_LASSO} is a convex problem \citep{boyd2004convex} and there exist a multitude of schemes for solving it numerically. Aside from generic convex solvers such as CVX \citep{cvx}, many numerical optimization methods have been developed: these include applications of the fast proximal gradient method of \cite{nesterov2013gradient} such as \citep{Beck2009,wright2009sparse}, of the Alternating Direction Method of Multipliers (ADMM) \citep{parikh2014proximal} such as~\citep{Afonso2010}, and of interior point methods~\citep{kim2007interior}. Recently, a quasi-Newton solver featuring substantial acceleration for high-accuracy solutions was devised by \citep{sopasakis2016accelerated}. 

In this paper, we consider $\Omega = \RR^M_+$, a constrained LASSO problem (with non-negative weights):
\begin{equation}\label{eq:consLASSO}
	\underset{\theta \in \RR^{M+1}_+}{\mathrm{minimize}} ~~ \frac{1}{2} \big\lVert \widehat{p} - \Phi \theta \big\rVert_2^2 + w \mathbf{1}^\top\theta, 
\end{equation}
which has a \emph{differentiable objective and very simple constraint set}.  For the adaptive case, sparsity can be improved using a scaled regularizer as described in \ref{A:sparsity}.  We implement a fast projected gradient method for this problem and use the log-barrier interior-point method (l1\_ls) based on the analysis in \citep{kim2007interior}. We set a logarithmic barrier for the non-negative constraints as $-\sum_{m=0}^{M}\log(\theta_m)$ and augment the objective function to obtain the associated centering problem 
\begin{equation}
	\underset{\theta \in \RR^{M+1}}{\mathrm{minimize}} ~~ \frac{z}{2} \big\lVert \widehat{p} - \Phi \theta \big\rVert_2^2 + zw \sum_{m=0}^{M} \theta_m  -\sum_{m=0}^{M}\log(\theta_i),
\end{equation}
where the centering problem becomes equivalent to the original as $z\to+\infty$. 

Post-processing methods that are geared towards de-biasing the solution and techniques for selecting the regularization parameter $w$ are presented in \ref{A:postProcessing} and \ref{A:regularization}, respectively.

\section{Recursive Estimation}
\label{sec:recursive}

The sparse density estimation methods that we have presented thus far implicitly assume that the travel times are all available for density estimation purposes. This is an inherent issue with traditional data analysis methods that naturally amount to \emph{offline} data processing . In order to capture real-time variation in travel time (for instance due to recurrent or non-recurrent events), this section presents an efficient online algorithm that operates directly on streaming measurements. Our approach is inspired by and extends \citep{freris2013recursive,RCS_allerton} and \citep{sopasakis2016accelerated} on \emph{recursive compressed sensing}, which applies LASSO to successive overlapping windows of the data stream.  

The key observation is that the dimensionality of our problem $M$ (size of $\theta$) does not depend on the size of the dataset $S$, but depends solely upon the granularity of time discretization (as well as the choice of scale parameters for M-L mixture component densities). For efficient sparse density estimation using streaming data, it is important to devise a method that (i) efficiently updates the Parzen density based on new measurements and (ii) provides fast numerical solutions to the LASSO problem, \eqref{eq:consLASSO}. Satisfying these two requirements ensures that the resulting method is suitable for an online implementation subject to high frequency streaming measurements and  stringent real-time constraints in estimating variable densities. To accomplish the second requirement, we propose using \emph{warm-starting} in solving \eqref{eq:consLASSO}, i.e., we use the previously obtained estimate $\widehat{\theta}$  as a starting point to an iterative LASSO solver (while properly updating the Parzen vector $\widehat{p}$). This is advantageous and leads to a substantial acceleration; see the experiments in \autoref{sec:recursiveTesting}.  We demonstrate how the first requirement can be satisfied by considering two scenarios: (i) \emph{sequential processing of travel times}, i.e., more data become available from the `same' underlying distribution, whence the changes in estimated parameter $\widehat{\theta}$ reflect enhancing the learning outcome based on new data, and  (ii) a \emph{rolling-horizon} setup, in which data are processed via windowing so as to track dynamic (within-day) variability in the travel time distributions in  real-time.  This can also be used for \emph{anomaly detection}, for instance, to identify incidents based on abrupt changes in the travel time distribution. We briefly discuss the two scenarios below. We assume, without any loss in generality, that the online algorithm accepts streaming travel time data and processes the data one observation at a time.

\subsection{Sequential Data Processing}
\label{sec:sequential}
We consider a stream of travel time data $\{T_1,T_2,\hdots\}$ and without loss of generality, we assume that they belong in the set $\{\tau_n\}_{n=0}^{N-1}$.  Sequential processing amounts to learning the underlying mixture densities corresponding to using the first $K+1$ data points based on the estimated mixture using the first $K$ data point, for $K\in \mathbb{Z}_+$. The $K$th LASSO problem is 
\begin{equation}
\widehat{\theta}^{(K)} \in \underset{\theta \in \Omega}{\mathrm{argmin}} ~~ \frac{1}{2} \big\lVert \widehat{p}^{(K)} - \Phi \theta \big\rVert_2^2 + w \mathbf{1}^\top\theta. 
\end{equation}
Observe that the matrix $\Phi$ \emph{does not depend on $K$}, but depends on our choice of time discretization (as well as the scale parameters). As explained above, we use warm-starting to obtain the solution $\widehat{\theta}^{(K+1)}$ while using as starting point to our numerical solver the previous solution $\widehat{\theta}^{(K)}$.
The Parzen density \eqref{eq:Parzen} is recursively updated as follows:
\begin{equation}
\widehat{f}^{(K+1)}(t) = \frac{K}{K+1}\widehat{f}^{(K)}(t) + \frac{1}{K+1} \kappa_h (t-T_{K+1}).
\end{equation}
Since we consider discretized data, the values $\kappa_h (t-t_j)$ can be precomputed for $t\in\{\tau_n\}_{n=0}^{N-1}$ and $j=0,1,\hdots, N-1$. Let us define the matrix $\Psi \in \RR^{N\times N}$ (depending exclusively on time discretization, where the $j$th column of $\Psi$, denoted by $\Psi_j$, is given by:
\begin{equation}\label{eq:P_matrix}
\Psi_j = [\kappa_h(\tau_0-t_j) ~ \hdots ~ \kappa_h(\tau_{N-1}-t_j)]^{\top}.
\end{equation}
Therefore, the vector $\widehat{p}^{(K+1)}$ can be obtained from $\widehat{p}^{(K)}$ along with the new data point $T_{K+1}$ using $O(N)$ operations as follows: 
\begin{equation}
\widehat{p}^{(K+1)} = \frac{K}{K+1}\widehat{p}^{(K)} + \frac{1}{K+1} \Psi_{T_{K+1}},
\end{equation}
where $\Psi_{T_{K+1}}\in \RR^N$ is the column of $\Psi$ corresponding to the (discretized) travel time $T_{K+1}$.

\subsection{Rolling-Horizon Data Processing}
\label{sec:rolling}
This recursive scheme allows real-time streaming data to be incorporated into the model as they arrive, and gradually removes old data that becomes irrelevant. This is achieved by sampling the input stream recursively via overlapping windowing, rather than using all historical data available to learn the model parameters.  This enables  the sparse density model to adapt to changes in the underlying data distribution (due, for example, to within-day variability in traffic conditions). 

We define $\mathbb{T}_W^{(j)}$ to be the $j$th window taken from the streaming travel time data of length $W$.  Without loss of generality, we will assume a fixed window of length $W$, and for a travel time data stream $\{T_1,T_2,\hdots\}$, we define $\mathbb{T}_W^{(j)} \equiv \{T_{j},\hdots,T_{j+W-1}\}$ and, similarly, $\mathbb{T}_W^{(j+1)} \equiv \{T_{j+1},\hdots,T_{j+W}\}$ to be two consecutive windows. 

Denoting the Parzen density corresponding to travel times in $\mathbb{T}_W^{(j)}$ by $\widehat{p}^{(j)}$, learning from the $j$th window amounts to solving
\begin{equation}
\widehat{\theta}^{(j)} \in \underset{\theta \in \Omega}{\mathrm{argmin}} ~~ \frac{1}{2} \big\lVert \widehat{p}^{(j)} - \Phi \theta \big\rVert_2^2 + w \mathbf{1}^\top\theta. 
\end{equation}
Noting the overlap between two consecutive windows, the $\{\widehat{\theta}^{(j)}\}$ sequence of parameters can be estimated recursively: this can be achieved by leveraging the solution obtained from the data in the $j$th window to warm-start 
the iterative solver for LASSO in window $j+1$.  The Parzen density \eqref{eq:Parzen} associated with travel time $t \in \RR_+$ corresponding to the $j$th window is given by 
\begin{equation}
	\widehat{f}^{(j)}(t) = \frac{1}{W} \sum_{l=j}^{j+W-1} \kappa_h (t-t_l).
\end{equation}
Thus, the empirical PW estimator can be viewed as a sliding empirical density estimator with a shifted kernel $\kappa_h(t-T_{j+W})$ being added for every successive window, while the outdated kernel $\kappa_h(t-T_{j})$ is removed, i.e.,
\begin{equation}
\widehat{f}^{(j+1)}(t) = \widehat{f}^{(j)}(t) + \frac{1}{W} [  \kappa_h (t-T_{j+W}) - \kappa_h (t-T_j)].
\end{equation}
Again, the vector $\widehat{p}^{(j+1)}$ can be obtained from $\widehat{p}^{(j)}$ very efficiently using 
$O(N)$ operations, as follows:
\begin{equation}
\widehat{p}^{(j+1)} = \widehat{p}^{(j)}(t) + \frac{1}{W} ( \Psi_{T_{j+W}}  - \Psi_{T_{j}}),
\end{equation}
where again $\Psi_{T_{j+W}},\Psi_{T_{j}}\in \RR^N$ are the columns of $\Psi$ corresponding to the (discretized) travel times $T_{j+W},T_{j}$, respectively.  Owing to the substantial overlap between consecutive data windows, the optimal solution to the $(j+1)$th problem is expected to be close to that of the previous problem. This leads to substantial acceleration in solving successive LASSO problems as demonstrated in \autoref{sec:recursiveTesting}.

\section{Experimental Validation}
\label{sec:testing}

In this section, we present numerical experiments that demonstrate the merits of our methods on real-life datasets. 

\subsection{Numerical Testing}
\label{ssec:numericalTest}
We first tested the performance of the proposed approach on a \emph{synthetic} dataset, using a known bi-modal probability density.  The example we consider compares the performance of a Gaussian mixture and the proposed mixture density using M-L functions. For this example, a data set of $S$ randomly drawn samples was used to construct the density estimate and a separate (out of sample) test data set of size $S_{\mathrm{test}}$ was used to calculate the  \emph{out-of-sample rooted-mean square error} ($\mathrm{RMSE}_{\mathrm{oos}}$) defined by
\begin{equation}
\mathrm{RMSE}_{\mathrm{oos}} = \sqrt{  \frac{1}{S_{\mathrm{test}}} \sum_{j=1}^{S_{\mathrm{test}}} \left( \widehat{f}(t_j) - \overline{f}(t_j)\right)^2  }. 
\end{equation}
The (true) density to be estimated is given by a mixture of two densities:  a Gaussian and a Laplacian with equal weights (0.5):
\begin{equation}
	f(t) = 0.5\frac{1}{\sqrt{200 \pi}} e^{-\frac{(t-60)^2}{200}} + 0.5 \frac{0.2}{2} e^{-0.2 |t-30|}.
\end{equation}
The density estimation was carried out using a data sample of size $S = 2000$, while the error is reported for an \emph{out-of-sample} dataset with $S_{\mathrm{test}} = 10,000$.  For travel times, we considered uniform per-second discretization of the interval $[1,300]s$, i.e., $M'=300$. The scale parameter $\sigma_m$ was allowed ten values $\{1,2,\hdots,10\}$  (therefore $M=10M'=3000$)  for both Gaussian and M-L mixture densities. For both cases, we set 
$N=2M'=600$, corresponding to uniform per-second discretization of the interval $[1,600]$ seconds.
For this example, all computations were performed using Matlab and CVX \citep{cvx} for numerical optimization. The test was performed ten times and average values are reported. 
A representative comparison between the density obtained using our proposed approach (for both Gaussian and M-L kernels), the PW density (using Gaussian kernel with variance $h=1.5$), and the true density is presented in Figure \ref{f1}. 
\begin{figure}[h!]
	\centering
	\resizebox{1.0\textwidth}{!}{%
		\includegraphics{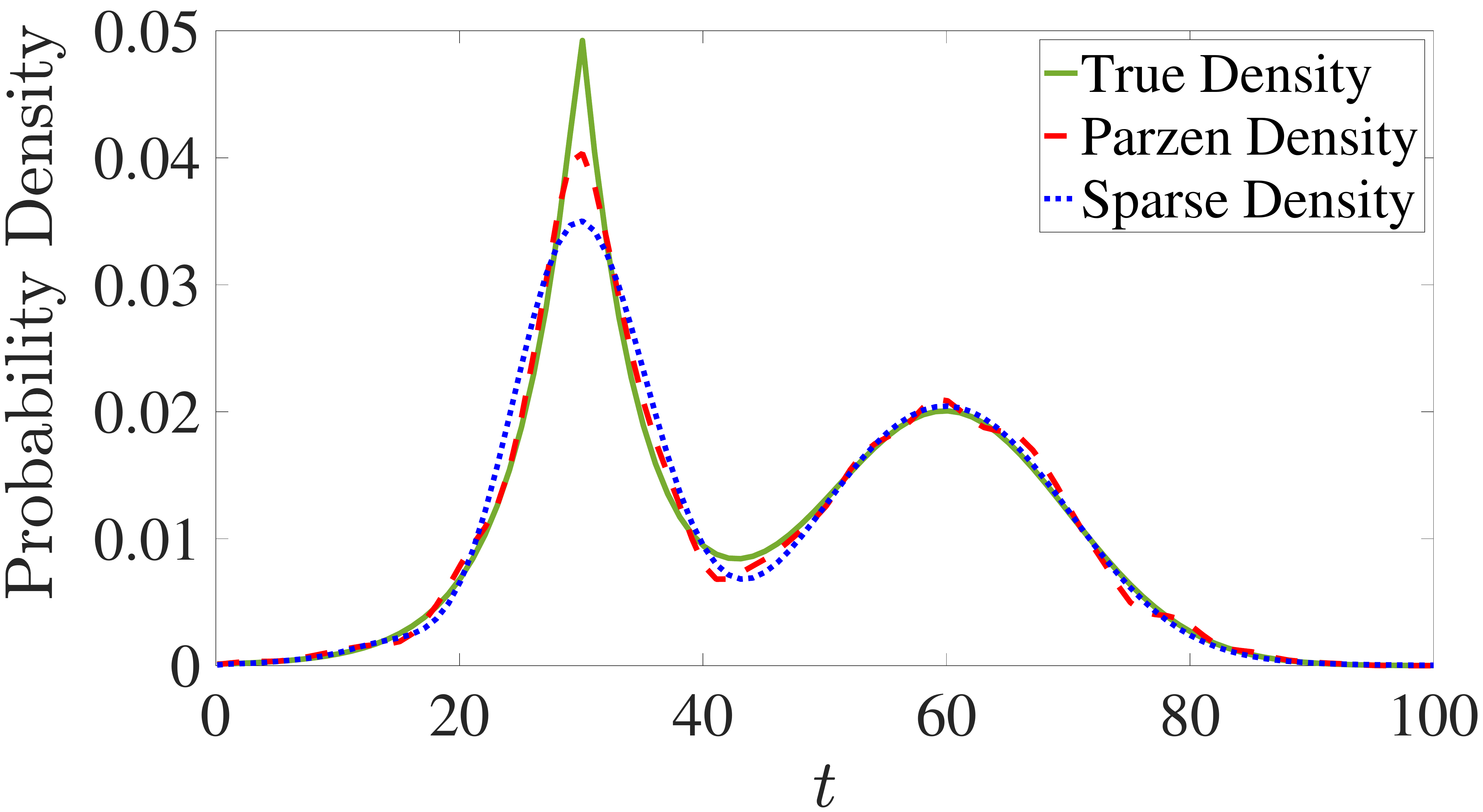}
		\includegraphics{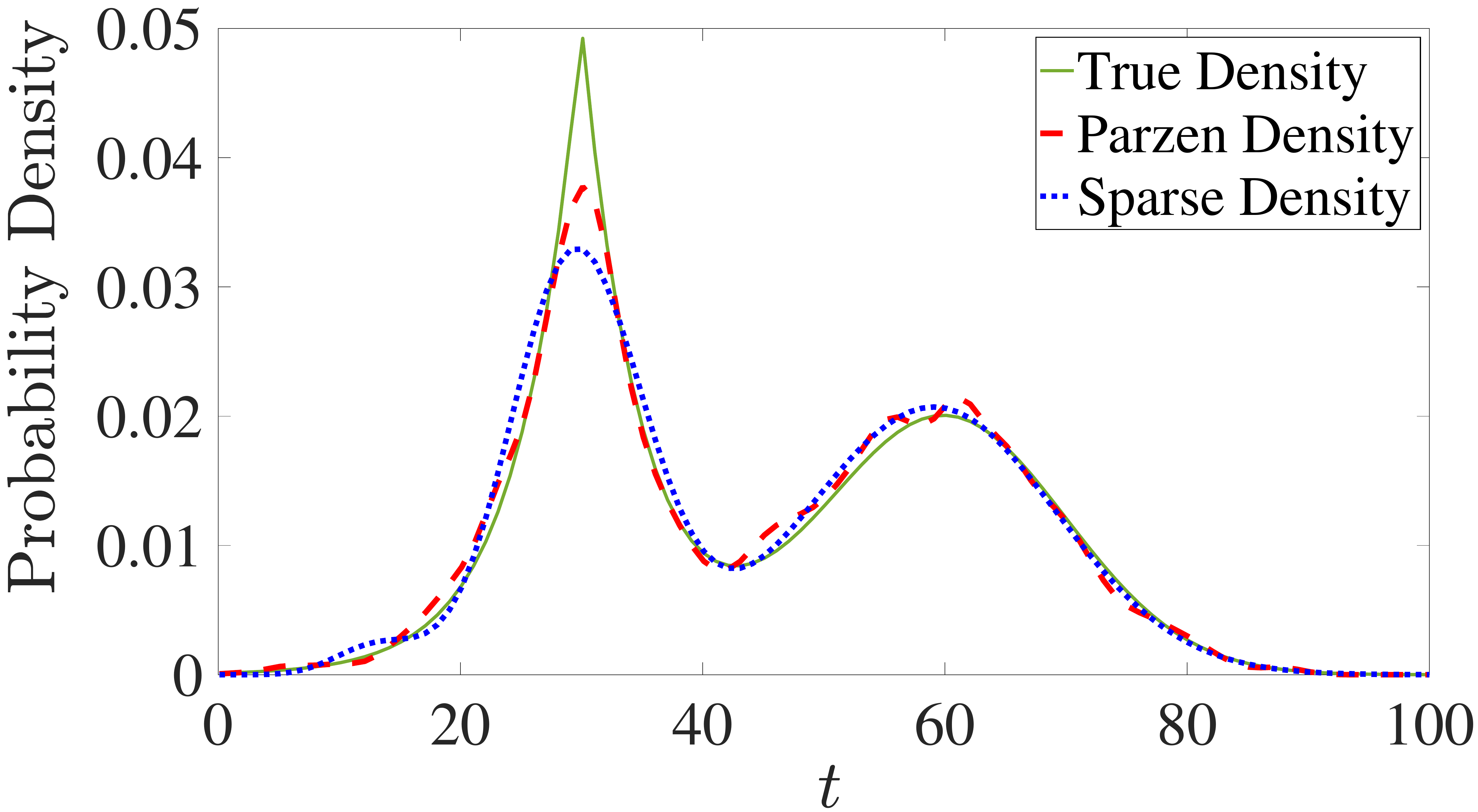}}
	
	(a) \hspace{2.6in} (b)
	
	\caption{True PDF vs. PW PDF vs. sparse mixture density PDF; (a) sparse density with Gaussian mixture component densities, (b) sparse density with M-L mixture component densities.} \label{f1}
	
\end{figure}

From this figure, it is evident that the sparse mixture estimators provide a very good fit to the true distribution, as is also shown in Table \ref{t1} (where $\mathrm{RMSE}_{\mathrm{oos}}$ is reported within $\pm1$ standard deviation).  The achieved sparsity was less than 5\% and 0.4\% of the sample sizes in the case of the Gaussian and M-L cases, respectively. This indicates that using M-L mixture densities promotes higher sparsity than using Gaussian mixture densities, i.e., higher compression rate, while at the same time achieving an order of magnitude improvement in goodness-of-fit, cf. Table \ref{t1}. In fact, the proposed sparse M-L estimator even outperformed PW in terms of accuracy, which perfectly demonstrates the superior fitting capabilities of our model. 
\begin{table}[h!]
	\centering
	\caption{Performance comparison on synthetic data.}
	\label{t1}
	\begin{tabular}[h!]{|ccc|}
		\hline
		Method & $\mathrm{RMSE}_{\mathrm{oos}}$ & Number of mixture components \\ \hline \hline
		PW estimator & 5.50e-04 $\pm$ 9.96e-05 & 2000 \\
		Sparse Gaussian estimator & 3.3e-03 $\pm$ 1.92e-05 & 95 \\
		Sparse M-L Estimator & 4.49e-04 $\pm$ 1.16e-04 & 7 \\ \hline
	\end{tabular}
\end{table}

\subsection{Experiments on Real Datasets}


\subsubsection{Dataset Description}
\label{ssec:realWorld}
The sparse mixture density estimation approach proposed was applied to travel times extracted from vehicle trajectories made available by the Next Generation SIMulation (NGSIM) program Peachtree Street dataset (\url{http://ngsim-community.org}). The arterial section is approximately 640 meters (2100 feet) in length, with five intersections and two or three through lanes in each direction. The section is divided into six intersection-to-intersection segments which are numbered from one to six, running from south to north. Of the five intersections, four are signalized while  intersection 4 is un-signalized.  The Peachtree Street data consists of two 15-minute time periods: 12:45PM to 1:00PM (noon dataset) and 4:00PM to 4:15PM (PM dataset). The dataset includes detailed individual vehicle trajectories with time and location stamps, from which the travel times of individual vehicles on each link were extracted.  In this study, the link travel time is the time a vehicle spends from the instant it enters the arterial link to the instant it passes the stop-bar at the end of the link (i.e., the time spent at intersections is excluded).

The second dataset we used contains vehicle trajectory data collected under the NGSIM program on eastbound I-80 in the San Francisco Bay area in April 2005. The study area is approximately 500 meters in length and consists of six expressway lanes, including a high-occupancy vehicle (HOV) lane and an on-ramp (see \cite{punzo2011assessment} for details). Using seven cameras mounted on top of a 30-story building adjacent to the expressway, a total of 5648 vehicle trajectories were captured on this road section in three 15-minute intervals: 4.00PM to 4.15PM; 5.00PM to 5:15PM; and 5:15PM to 5.30PM. These periods represent the build-up of congestion, the transition between uncongested and congested conditions, and full congestion during the peak period, respectively.

\subsubsection{Fitting Results and Comparisons}
\label{sssec:results}
In order to demonstrate the effectiveness of the proposed approach, we have chosen to estimate the travel time distributions of southbound traffic on the signalized arterial links along Peachtree Street for the two time periods. We used Gaussian component densities for the empirical distribution $\widehat{p}$ (the PW density), where the  bandwidth $h$ was calculated according to the (standard) approximation proposed by \cite{silverman1986density}: $h=1.06\varsigma S^{-1/5}$ is picked to minimize the integral mean-square error (where $\varsigma$ is the sample variance and $S$ is the sample size). For the M-L mixture, we used $M'=300$ location parameters with scale parameters in the set $\sigma_{m} \in \{1,2,3,4,5\}$ (i.e., $M=1500$ mixture components were used in the estimation procedure).

Figure \ref{f3} (a) shows the PW PDF and the estimated sparse PDF (using M-L functions) for the travel times of the southbound vehicles during the noon period. The fitted distribution is clearly bi-modal and closely follows the PW PDF. The bi-modality of the travel time distribution can be attributed to the presence of two traffic states:  non-stopped vehicles along the entire corridor in the southbound direction and stopped vehicles experiencing delay at one or more of the signals. Observe that while the number of mixture components required to calculate the PW density is equal to the number of data samples (58 for this case), the proposed estimation algorithm achieves a similar accuracy with a \emph{much sparser representation}: only four M-L mixture components were needed; i.e., a compression rate of about 15:1.  

We compared our approach against the \emph{Expectation Maximization (EM)} algorithm~\citep{bishop2006pattern}, the prevalent method for estimation of Gaussian mixture models~\citep{wan2014prediction}. 
The EM algorithm (using Gaussian mixtures) has been widely used for the estimation of travel time densities,  despite its slow rate of convergence \citep{wu1983convergence,archambeau2003convergence}, and the dependence of the parameter estimates on the choice of the initial values \citep{biernacki2003choosing}. The commonly adopted method to prevent the EM algorithm from getting trapped in local minima is to start the algorithm with different initial random guesses \citep{wan2014prediction}. The importance of properly defining the stopping criterion to ensure that the parameters converge to the global maximum of the likelihood function has been highlighted in \citep{karlis2003choosing,abbi2008analysis}. In all our experiments, we used ten randomly selected initial estimates; for termination criterion, we used tolerance threshold (selected as $10^{-3}$) on the absolute difference between two successive root-mean squared error (RMSE) estimates, where  
\begin{equation}
\mathrm{RMSE}= \sqrt{  \frac{1}{N} \sum_{j=1}^{N} \left( \widehat{p}(t_j) - \overline{p}(t_j)\right)^2  }. 
\end{equation}

A known issue with the EM algorithm is that it requires predetermining the number of mixture components. This is in contrast to our method, which \emph{optimally determines the number of mixture components concurrently with the fitting procedure}. Given the number of mixture components, the EM algorithm is an iterative  method used to estimate the mean and variance of each Gaussian mixture density, along with the weight vector $\theta$. \emph{Note that the EM algorithm solves for maximum-likelihood estimates of the mixture distribution parameters; it does not minimize the RMSE}. Figure \ref{f3} and Table \ref{t2} summarize the results. 
\begin{figure}[h!]	
	\centering
	\resizebox{1.0\textwidth}{!}{%
		\includegraphics{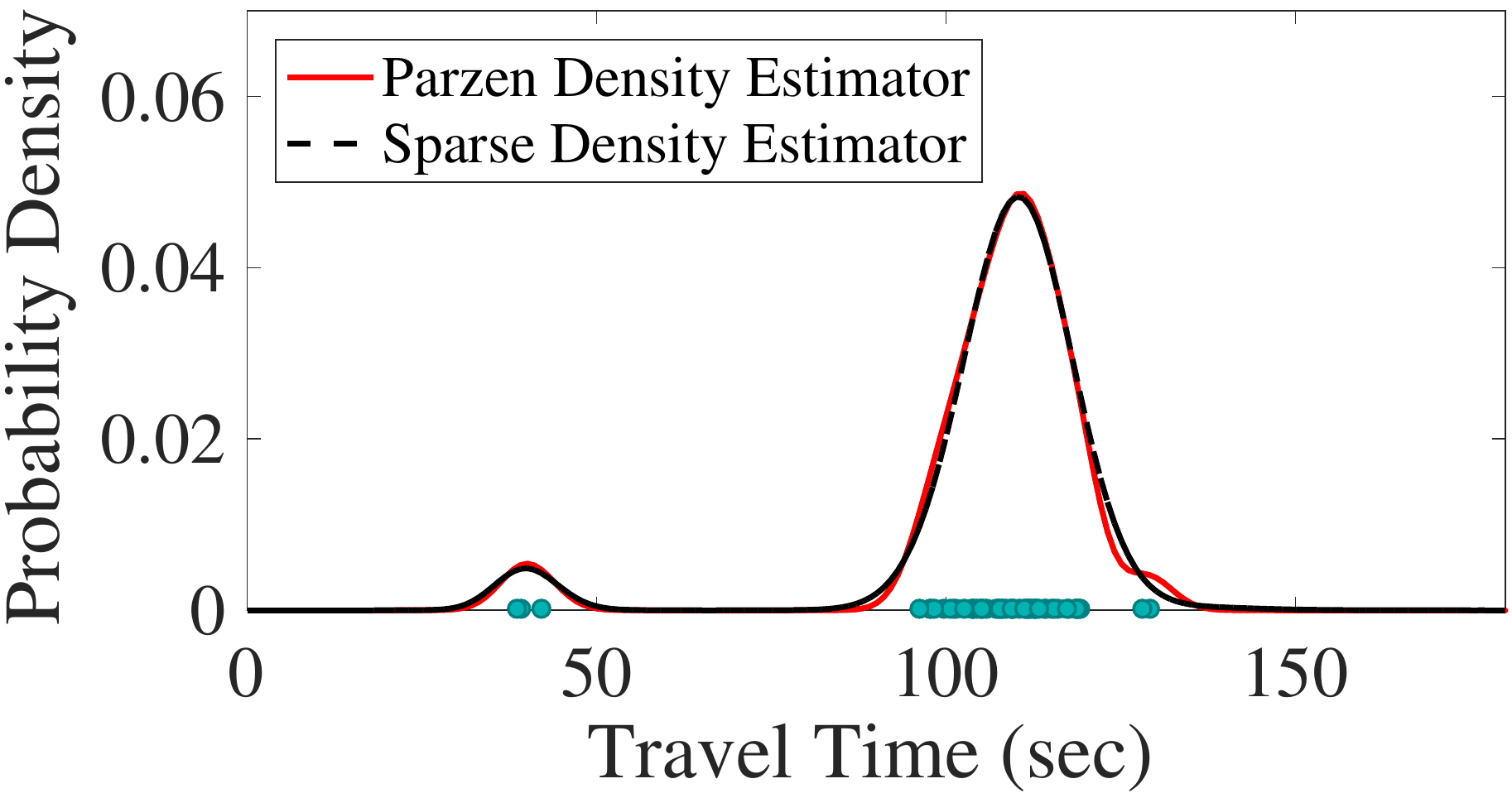} \hspace{0.3 in}
		\includegraphics{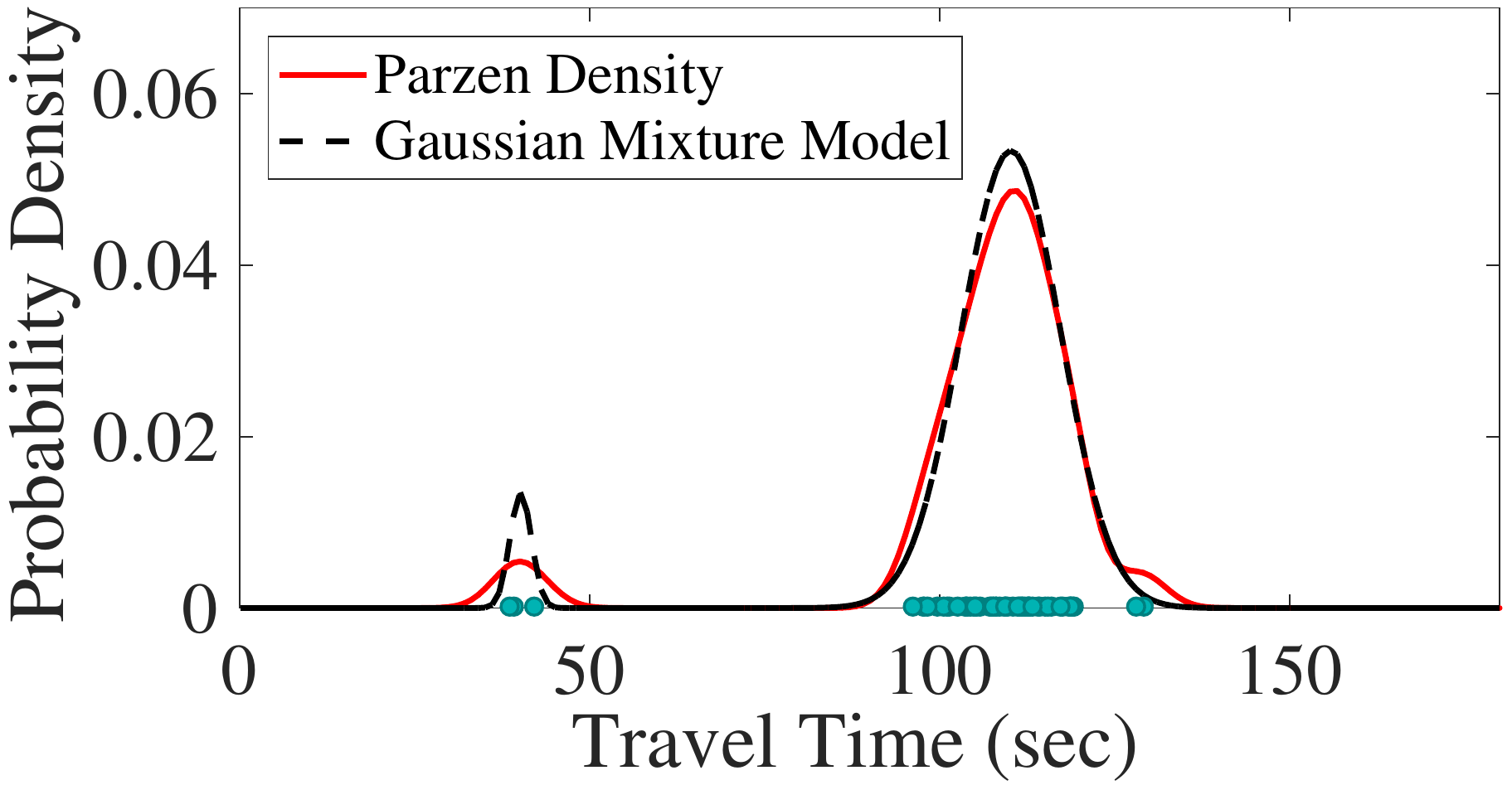}}
	
	(a) \hspace{2.5in} (b)
	
	\resizebox{1.0\textwidth}{!}{%
		\includegraphics{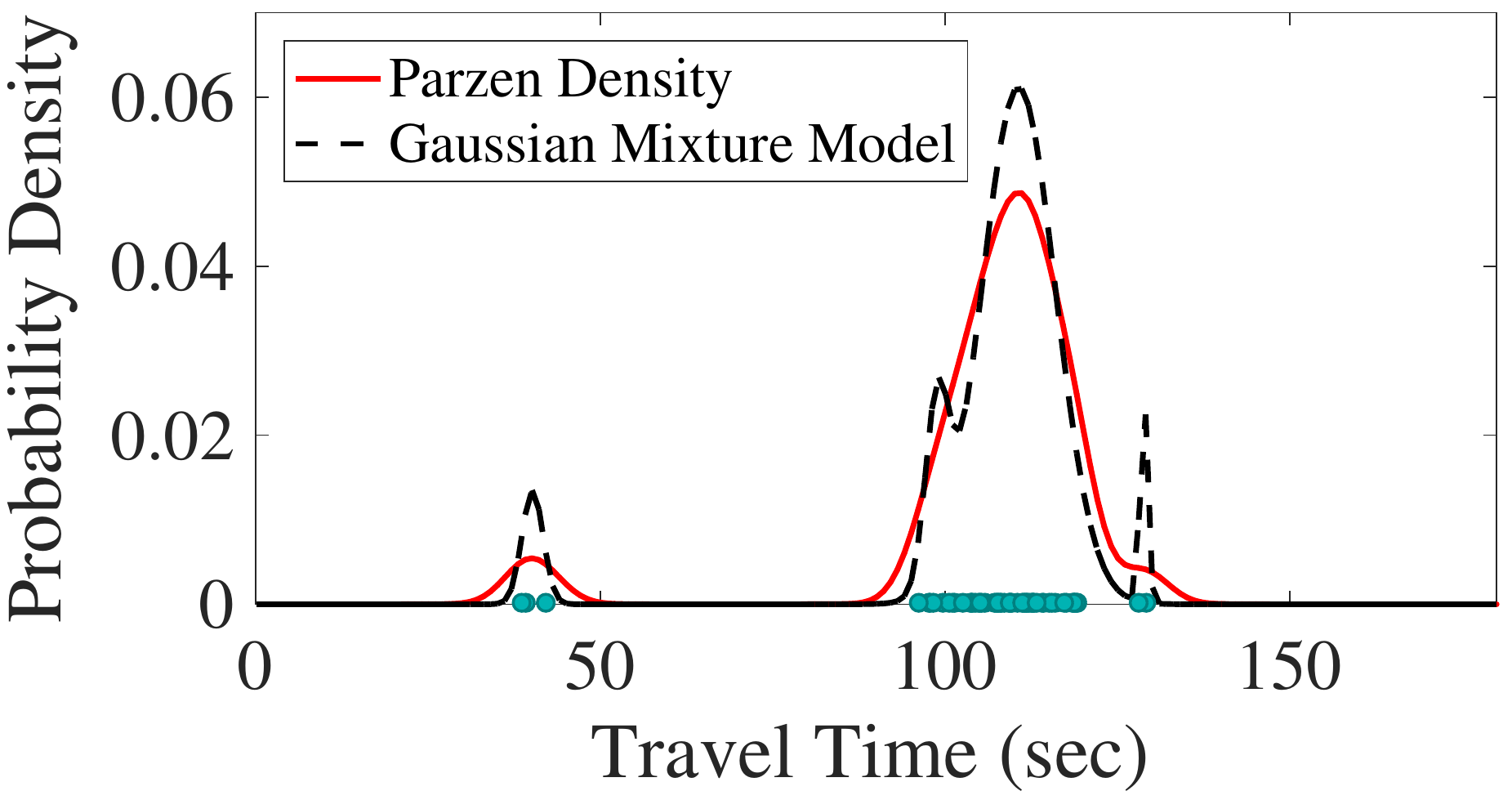} \hspace{0.3 in}
		\includegraphics{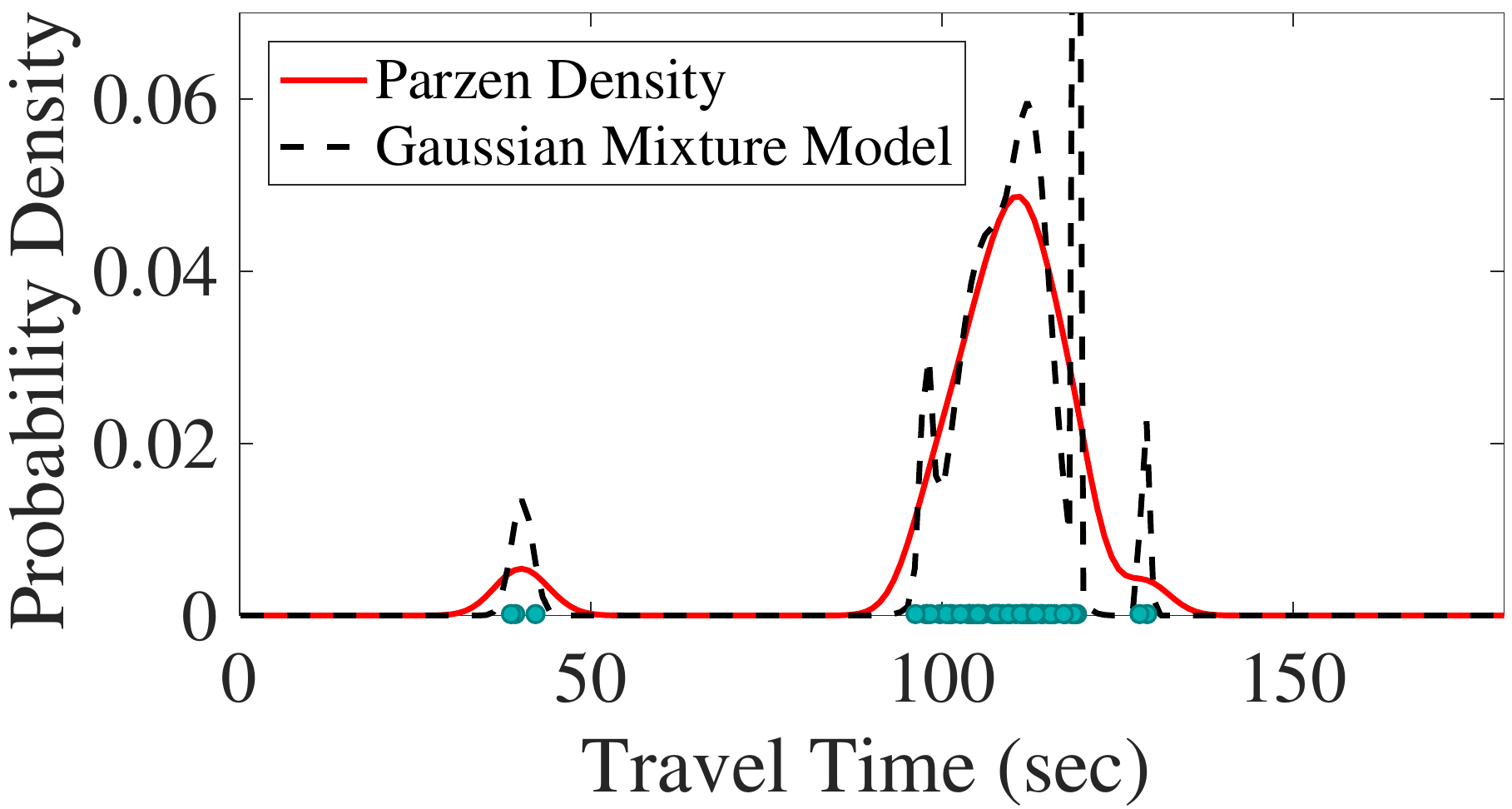}}
	
	(c) \hspace{2.5in} (d)
	
	\caption{Travel time densities of Peachtree Street (southbound, noon) depicting the locations of the travel time samples (green circles along the horizontal axis); (a) M-L mixture densities vs. Parzen density, (b) Gaussian mixture with two modes vs. Parzen density, (c) Gaussian mixture with four modes vs. Parzen density, (d) Gaussian mixture with six modes vs. Parzen density.} \label{f3}
\end{figure}
\begin{table}[h!]
	\centering
	\caption{Performance comparison: M-L vs. EM.}
	\label{t2}
	\begin{tabular}[h!]{|cccc|}
		\hline
		Method &  No. mixture components & RMSE & Log-likelihood \\ \hline \hline
		Sparse M-L Estimator & 4 & 0.0004  & N/A \\
		EM & 2 & 0.0009 & 0.0021  \\
		EM & 4 & 0.0012 & 0.0029 \\
		EM & 6 & 0.0063 & 0.0152 \\ \hline
	\end{tabular}
\end{table}
The optimal sparse fitting contains four M-L mixture components and we also tested the EM algorithm with two, four and six Gaussian mixture components. Increasing the number of components in the EM algorithm increases (i.e., improves) the log-likelihood but the RMSE tends to get worse beyond two mixture components.  This is indicative of the EM algorithm's tendency to \emph{over-fit to artifacts in the data} with larger numbers of mixture components.  This is indicative of a susceptibility to data errors of the EM algorithm. (This is a well-known weakness of log-likelihood maximization as opposed to least-squares estimation.)  In contrast, our model has the favorable property that the goodness-of-fit typically \emph{increases} with the number of mixture components used.  Figure \ref{f2_R2} illustrates this using travel times from another dataset (namely, I-80): we evaluated the RMSE for our sparse density estimator vs. the EM algorithm with varying numbers of mixture components (for M-L component densities, we varied the regularizing parameter $w$ so as to achieve different sparsity levels).
\begin{figure}[h!]
	\centering
	\resizebox{0.5\textwidth}{!}{%
		\includegraphics{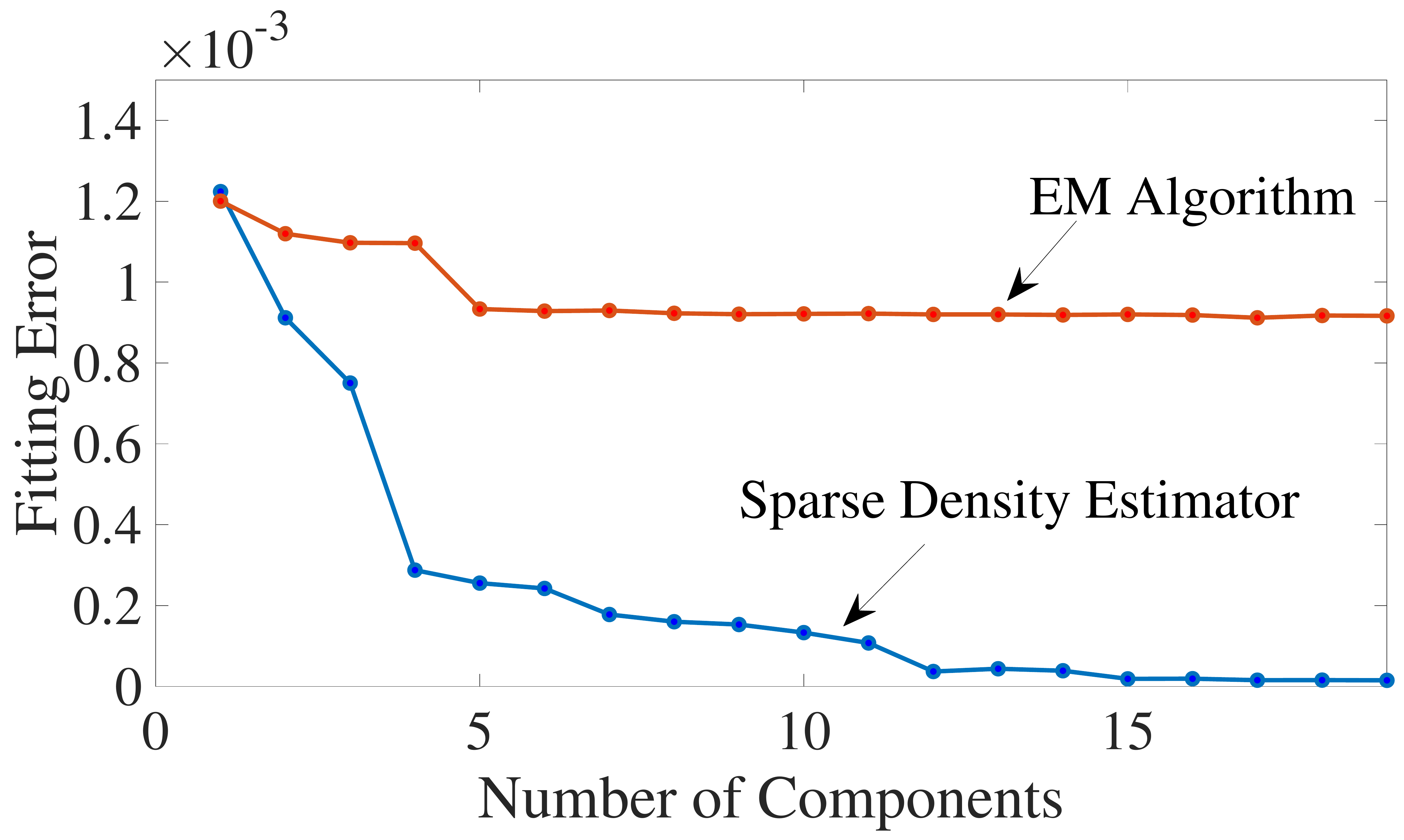}}
	\caption{Fitting accuracy of Sparse Density Estimator vs EM
		for variable number of components; dataset: I-80.} \label{f2_R2}
\end{figure}

\subsection{Inference with Parsimonious Models}
In order to highlight the predictive capabilities and interpretability of parsimonious models, we have tested our method on \emph{hold-out} real data from the I-80 dataset:  We divided the bulk of the I-80 data in two parts (corresponding to different timestamps ): (i) a training dataset and (ii) a hold-out test dataset (where we selected a ratio of $4:1$ for training vs. test data). We then fit our model using the training data and tested its performance (measured via goodness-of-fit) on the hold-out test data. It is worth noting that this scenario is a challenging one due to the heterogeneity of the travel times recorded over intervals of variable traffic conditions. The results are reported in Figure \ref{fig1}: 
Figure \ref{fig1} (a) plots the PW on the training and hold-out data, along with the sparse density obtained using M-L mixture densities (12 mixture components were used by our sparse density estimator in this case); Figure \ref{fig1} (b) plots the fitting error (RMSE) for both our method and the EM algorithm using a varying number of mixture components, namely 1-12. It is evident from this experiment that our method clearly outperformed the EM algorithm in terms of higher fitting accuracy on hold-out data.
\begin{figure}[h!]
	\centering
	
	\resizebox{1.0\textwidth}{!}{%
		\includegraphics{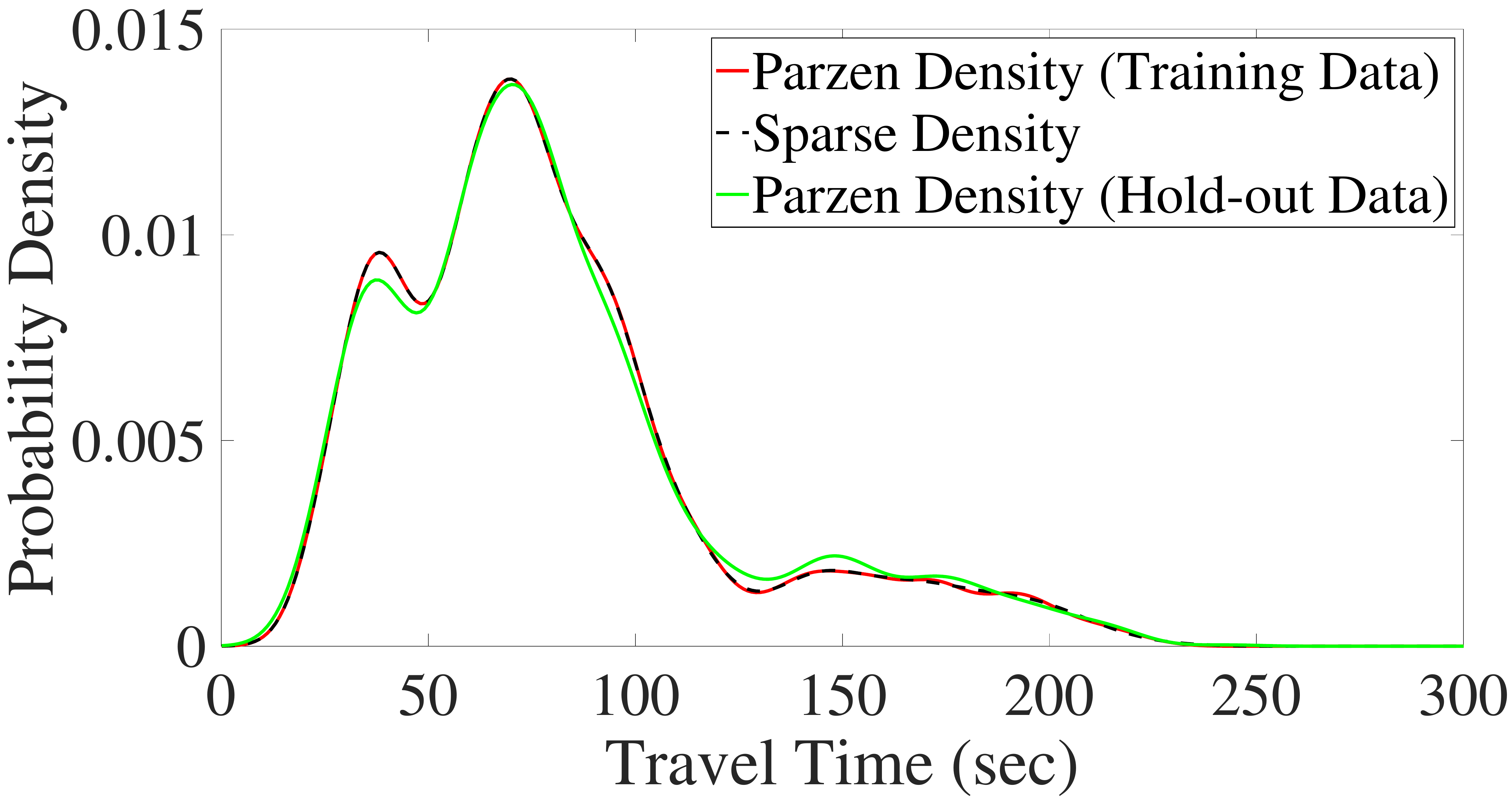}
		\includegraphics{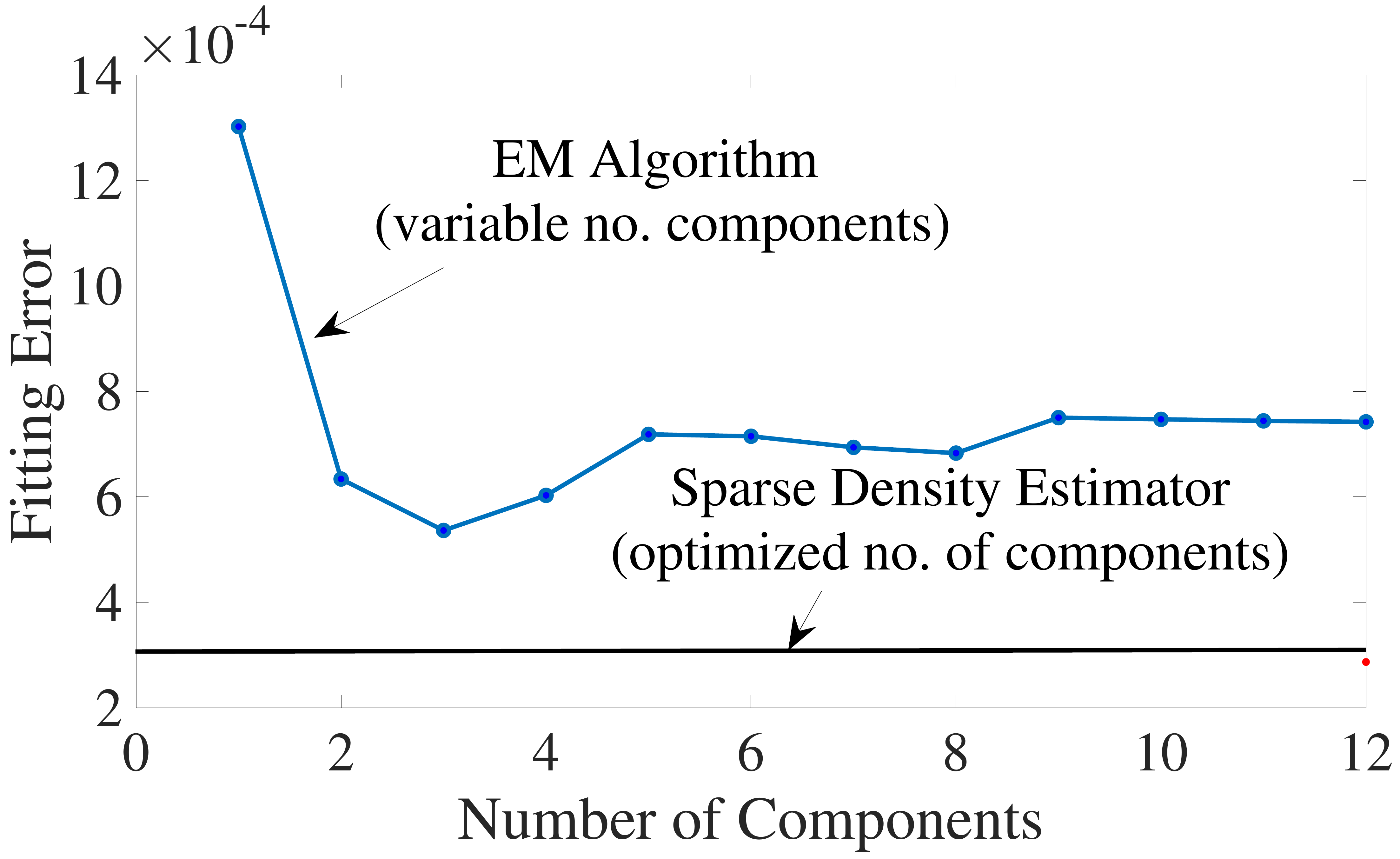}}	
	
	(a) \hspace{2.6 in} (b)
	
	
	
	\caption{(a) Parzen density for training data, hold-out data, and fitted sparse density, (b) Fitting error on hold-out data: M-L sparse density estimator (straight line in black) vs EM algorithm for variable number of mixture components (blue); dataset: I-80.} \label{fig1} 
\end{figure}


We tested our method vs. $\ell_2-$regularization on the Peachtree (northbound, noon) dataset. For both methods, we chose  $M=1500$ M-L mixture components for model selection ($M'=300$ and a scale parameter set $\sigma_{m} \in \{0.2,0.3,0.5,1,1.5\}$). For $\ell_2$-regularization, the value $\widetilde{w}$ was selected from the set $\{5\cdot10^{-5}, 5\cdot10^{-4}, 5\cdot10^{-3}, 5\cdot10^{-2}, 5\cdot10^{-1}\}$ by $5:1$ cross-validation. 
\begin{figure}[h!]
	\centering
	\resizebox{1.0\textwidth}{!}{%
		\includegraphics{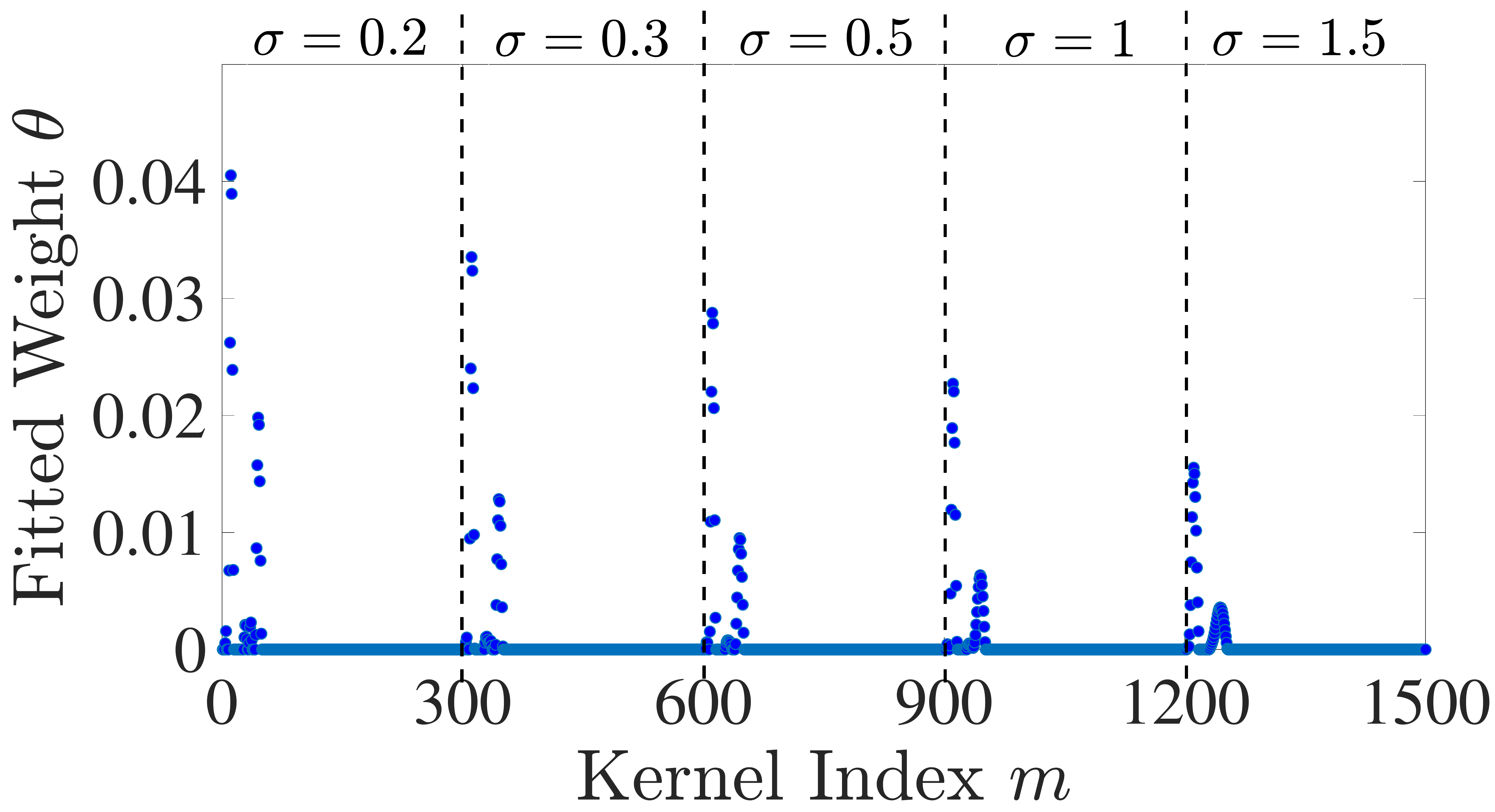}
		\includegraphics{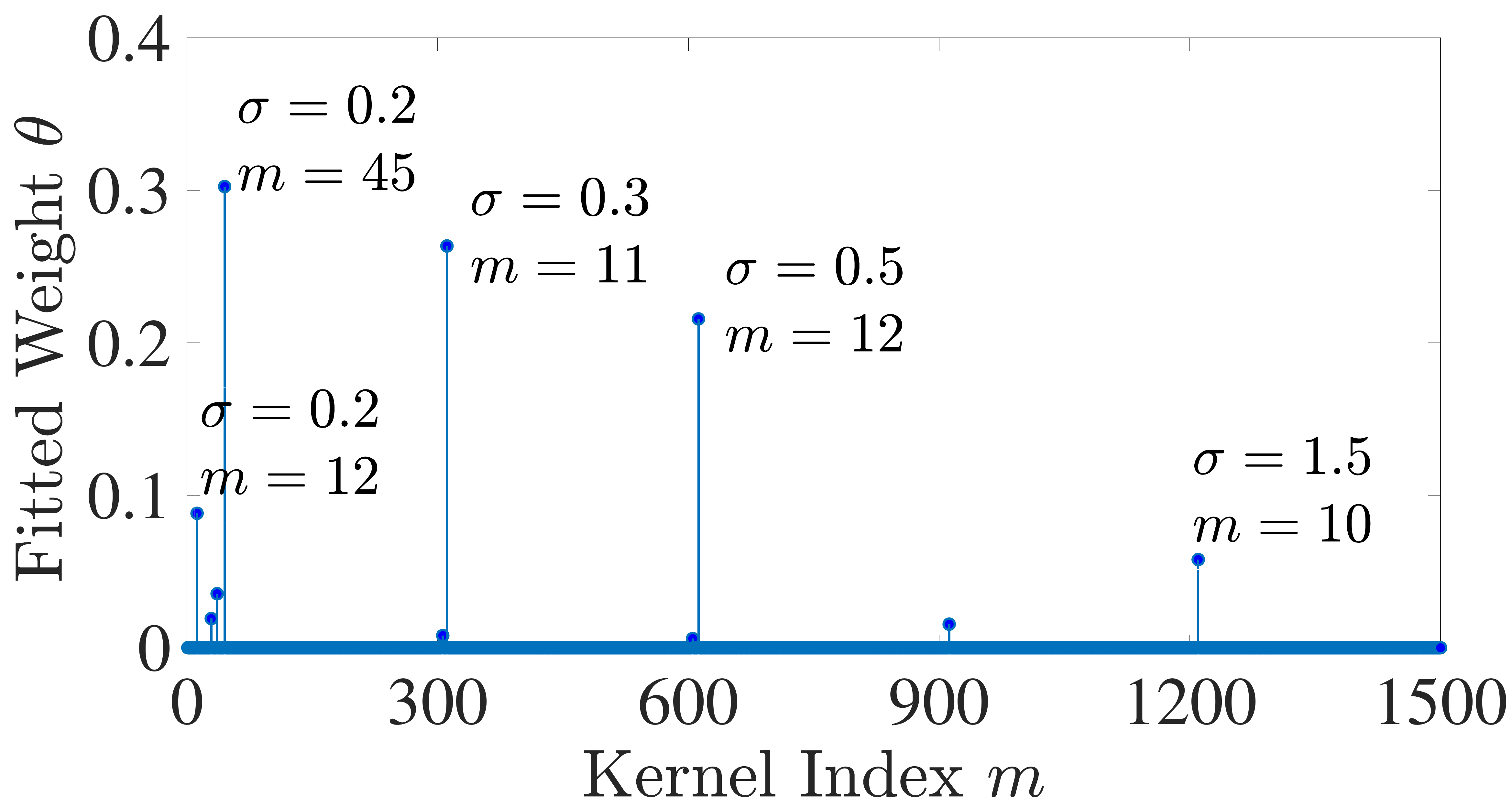}}
	
	(a) \hspace{2.6 in} (b)
	
	
	
	\caption{Weight vector $\theta$  (a) using  $\ell_2$-regularization (b) using sparse density estimation ($\ell_1-$regularization); dataset: Peachtree (northbound, noon).} \label{fig3}
\end{figure}
Figure \ref{fig3} illustrates the results.  Both methods achieved an RMSE of about $0.008$. Nonetheless, the number of mixture components (and corresponding weights) that need to be stored to re-create and predict the travel time distribution was substantially reduced to only 5 M-L mixture densities using sparse density estimation (from 84 needed for $\ell_2-$ regularization).  
In addition to reduced storage requirements, the sparse density estimate allows for making inference with ease about the underlying data through the selected mixture components and their corresponding weights. For instance, the selected M-L components indicate that the underlying travel time data can be approximated well by two peaks located at around $t=11$ and $t=45$. On the other hand, the mixture components selected by the $\ell_2$-norm regularization are much less informative. This parsimony is further illustrated in Figure \ref{f4} where the experiment was conducted on the I-80 dataset.
\begin{figure}[h!]
	\centering
	\resizebox{0.5\textwidth}{!}{%
		\includegraphics{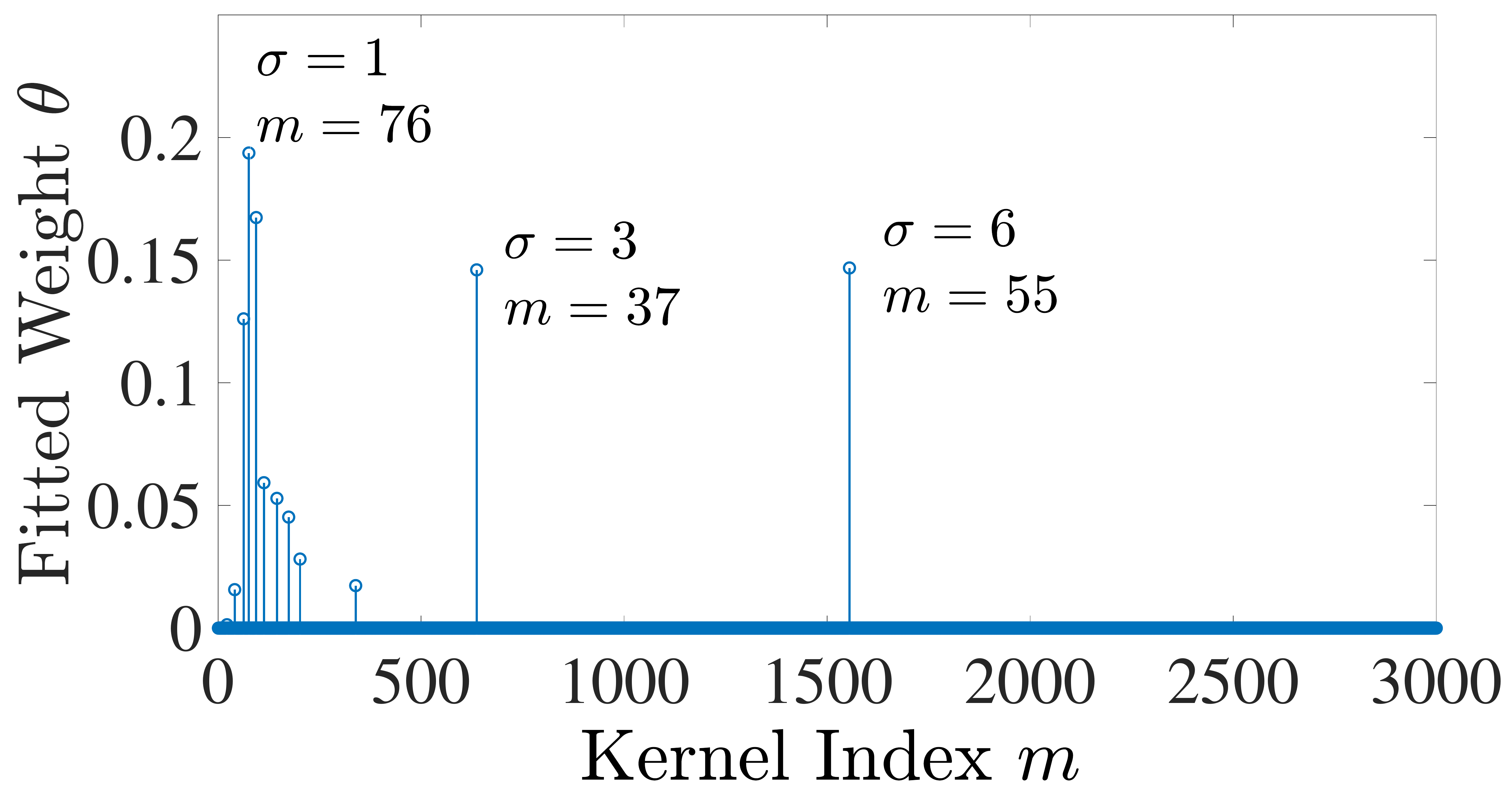}}	
	\caption{Sparse density estimation, I-80.} \label{f4}

\end{figure}

\subsection{Merits of Mittag-Leffler Mixture Densities}
In this section, we demonstrate the superiority of the adaptive approach with M-L mixture densities over the non-adaptive (Gamma mixture densities with a single-scale parameter $\sigma$) in terms of parsimony. For this case study, we considered the travel time distribution of the northbound traffic along Peachtree street in the noon time period. The sparse density estimation was first carried out using the M-L mixture densities with $\sigma_m \in \{1,2,3,4,5\}$ and then using Gamma mixture densities with single parameter  $\sigma = 1$. The solutions are depicted in Figure \ref{fig6_R1}(a) and Figure \ref{fig6_R1}(b) respectively, where we have used $M= 1500$ ($M'=300$ uniform per-second discretization)  for the M-L mixture densities and $M=300$ for the Gamma mixture densities. 
\begin{figure}[h!]
	\centering
	\resizebox{1.0\textwidth}{!}{%
		\includegraphics{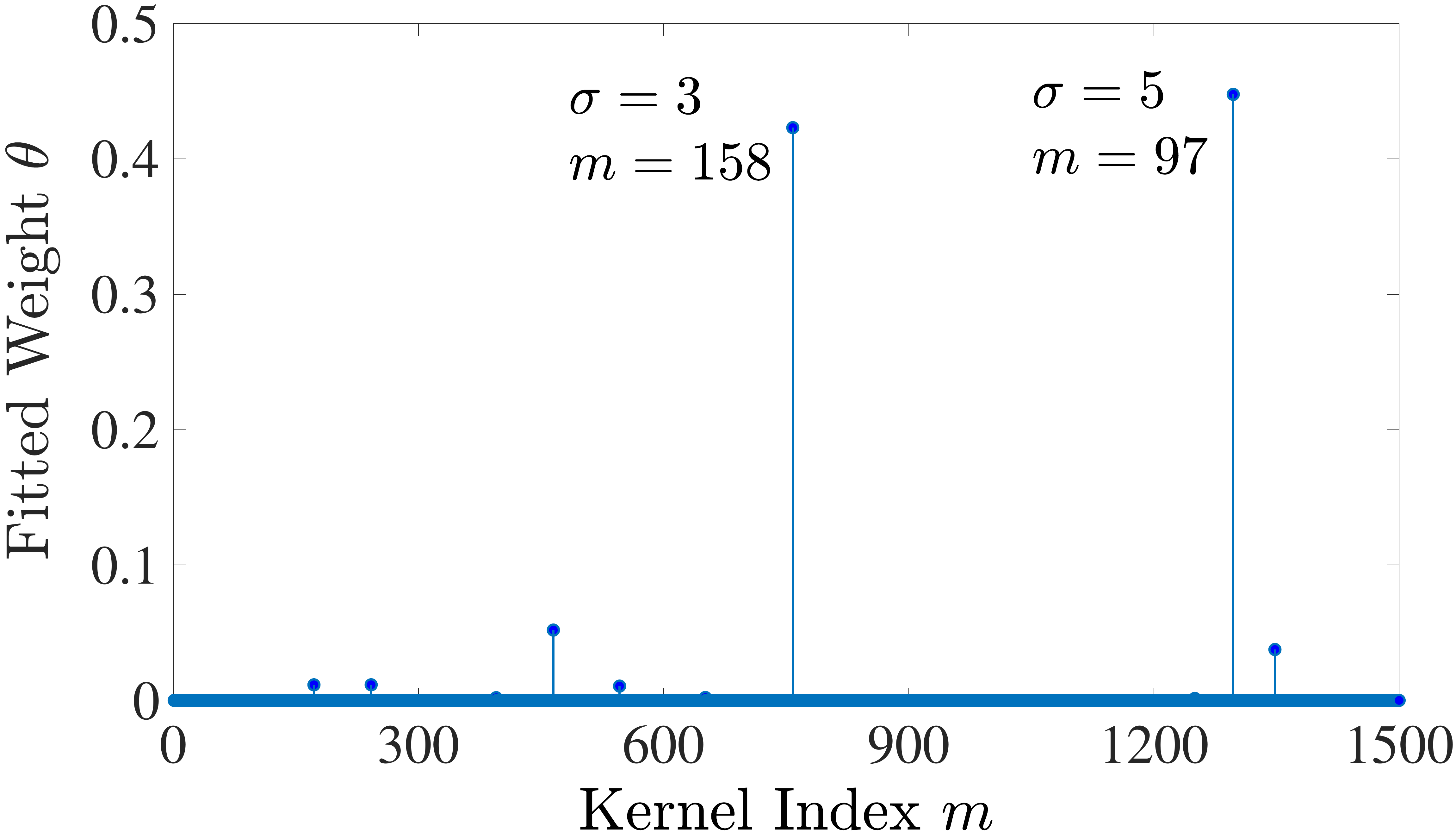}
		\includegraphics{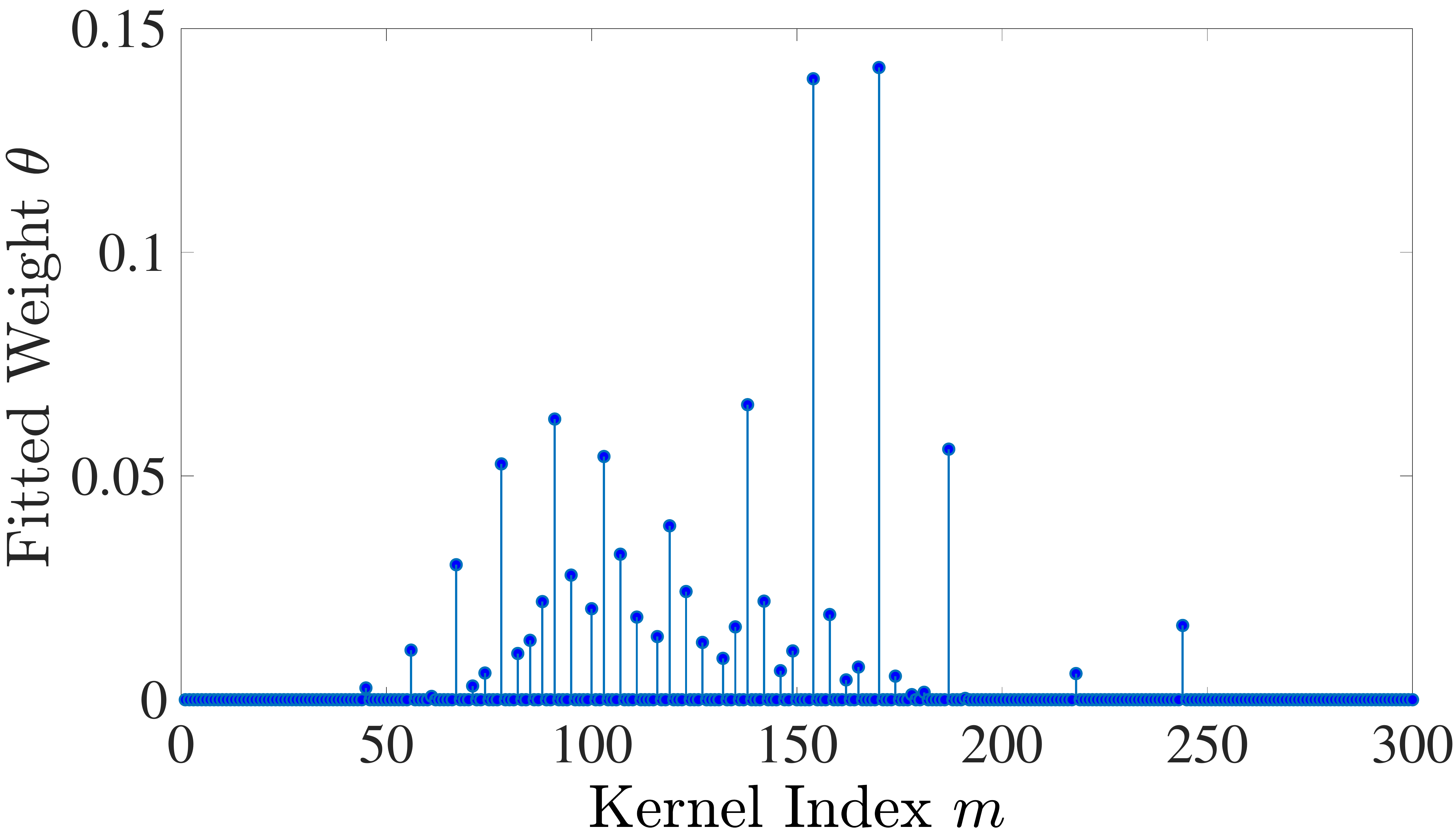}}
	
	(a) \hspace{2.6 in} (b)
	
	
	
	\caption{Weight vector $\theta$ for (a) M-L mixture densities (b) Gamma mixture densities; dataset:  Peachtree (northbound, noon).} \label{fig6_R1}
\end{figure}
The figures  indicate that the travel time density can be efficiently represented using two dominant modes (with different scale parameters). However, in the case of the Gamma mixture, a much larger number of components was required. Although using a $\sigma = 5$ reduces the number of Gamma mixture components required to 2, the sparse Gamma estimate cannot accurately capture the shape of the distribution, as shown in Figure \ref{fig5}(a); in contrast, the estimated M-L mixture is indistinguishable from the PW density, as depicted in Figure \ref{fig5}(b). 
\begin{figure}[h!]
	\centering
	\resizebox{1.0\textwidth}{!}{%
		\includegraphics{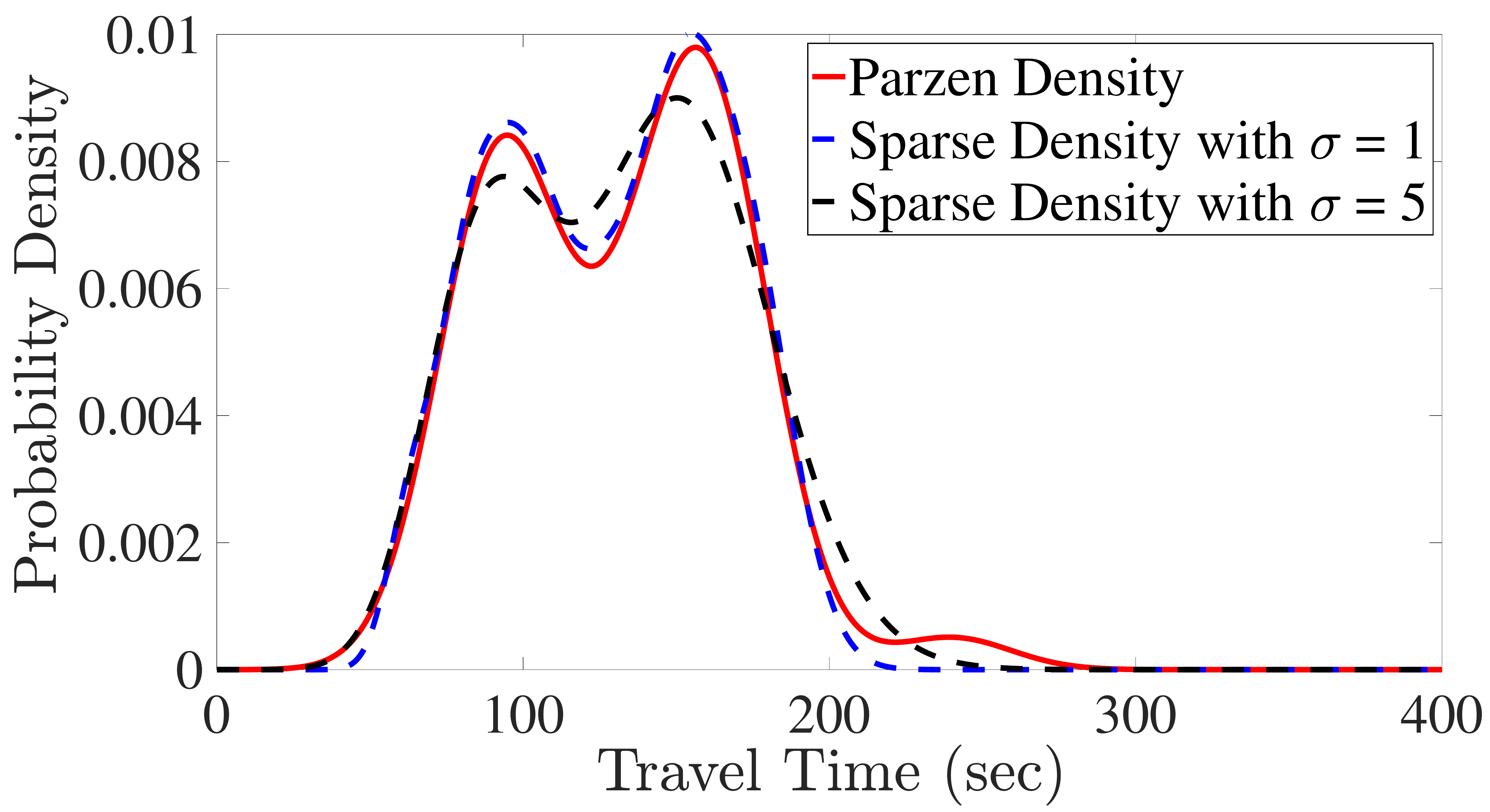}
		\includegraphics{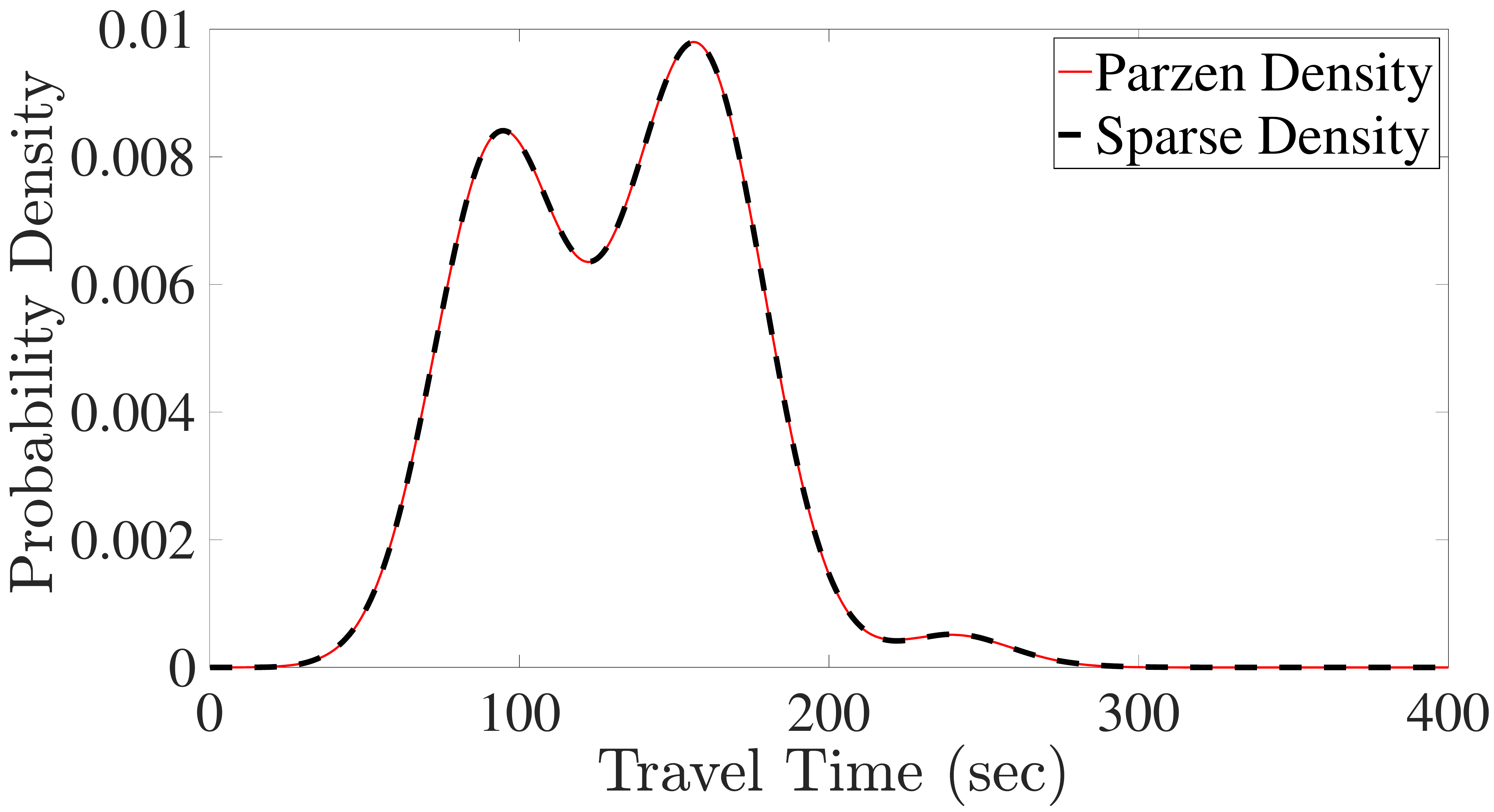}}
	
	(a) \hspace{2.6 in} (b)
	
	
	\caption{ Travel time densities of Peachtree Street (northbound, noon): (a) Gamma mixture, (b) M-L mixture.} \label{fig5}
\end{figure}

\textbf{Interpreting the results}.
From the weight vector of the M-L mixture in Figure \ref{fig6_R1}(a), it is clear that the predominant mixture components associated with the highest weights are the M-L densities with $\sigma = 5$ located at $t=97$ seconds, and $\sigma = 3$ located at $t=158$.  From this alone, we can infer the most likely travel times of the northbound (noon) traffic along Peachtree street, whereas the weight vector associated with the Gamma mixture is not quite as informative.

\subsection{Real-World Testing of Recursive Algorithm}
\label{sec:recursiveTesting}
The recursive algorithm on streaming data was tested using the I-80 dataset.  We track the changes in the travel time density on I-80 using the recursive algorithm, by taking a fixed window size of $W=100$ travel time samples for each instance of sparse density estimation (along with parameters $M'=300, N=600$ corresponding to per-second uniform discretization and scale parameters $\sigma_m \in \{1,2,3,4,5\}$, whence $M=1500$ M-L mixture components are considered). By processing the newly arriving samples one at a time (and simultaneously discarding the oldest ones), the density is constantly updated with time following the mechanism presented in \autoref{sec:rolling}.  The travel time densities for the PM peak period predicted by the recursive algorithm are depicted in Figure \ref{f10}, where we can observe that the number of modes, as well as their locations, vary significantly over time. 
\begin{figure}[h!]
	\centering
	\resizebox{0.5\textwidth}{!}{%
		\includegraphics{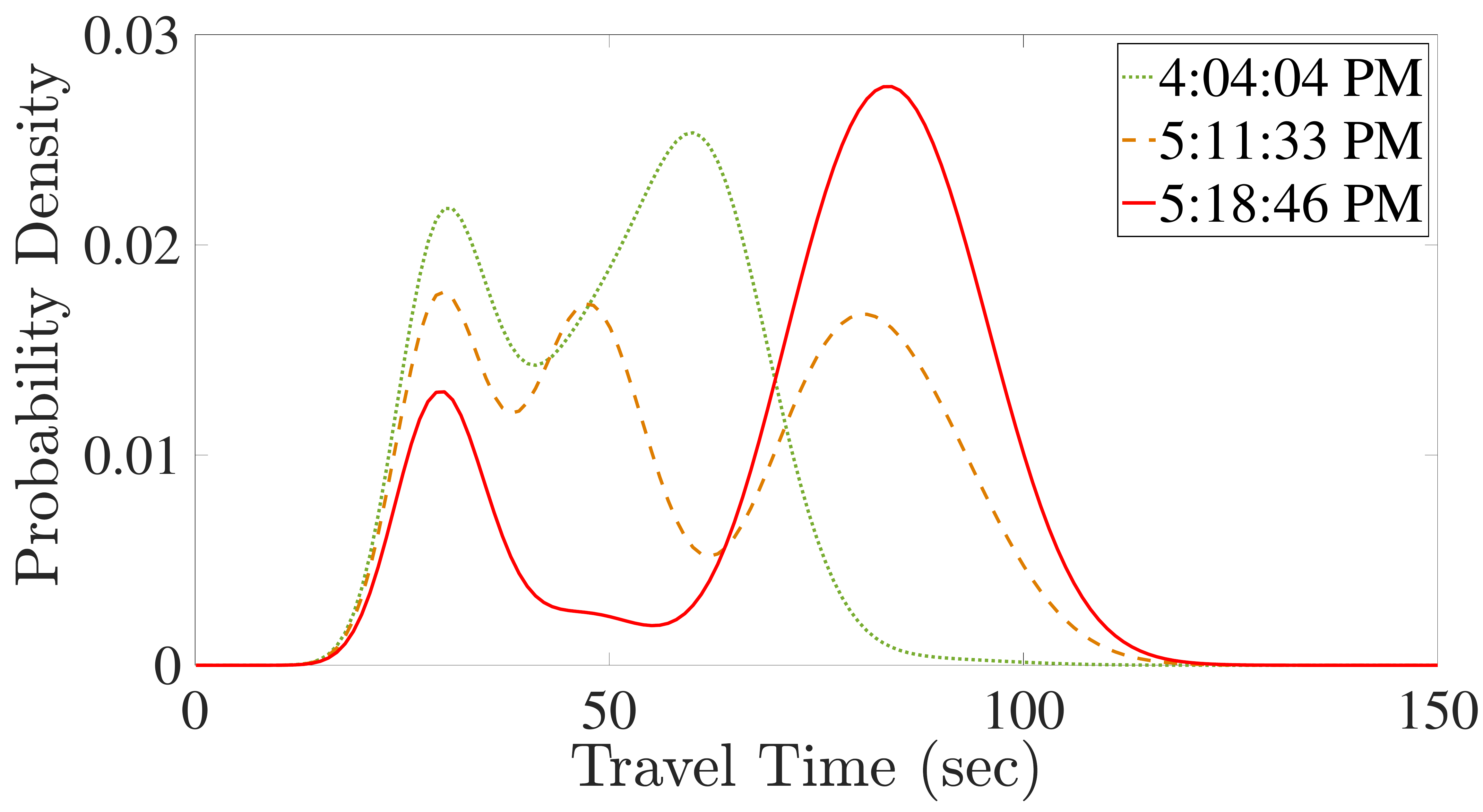}}
	\caption{Time varying travel time density on I-80 (eastbound, PM).} \label{f10}
\end{figure}

For the first time period under consideration, the travel time density at (a representative) time of  4:04PM is plotted; clearly, the density can be captured by a bi-modal distribution. This corresponds to the uncongested period where the travel times of nearly all the vehicles are below 80 seconds. However, at about 5:08PM (which represents the time when congestion begins to build up), the number of modes increases to 3, introducing a new cluster of vehicles with travel times between 70 and 120 seconds. After congestion has set in, the number of modes again reduces to 2 in the third time-period, and the locations of these modes indicate that the travel times of all vehicles have increased. In brief, these results highlight the capability of the recursive algorithm to track the varying travel time density in real-time, in a means that is also robust to the variations encountered by  individual vehicles. The model parameters estimated by the recursive algorithm reflect the underlying traffic conditions, and can capture the multi-modality in these distributions very efficiently. 

The run-time was reported to be just over 2.5 minutes for recursive estimation vs. about 2.5 hours using the standard method (non-recursive one). This experiment solidifies our claim for the feasibility of a truly real-time implementation of our methods (note that a run-time of 2.5 minutes was needed to track the variability over an interval of 45 minutes).   A series of snapshots illustrating the dynamic variation of densities is given in Figure \ref{f11}.
\begin{figure}[h!]
	\centering
	\resizebox{0.65\textwidth}{!}{%
		\includegraphics{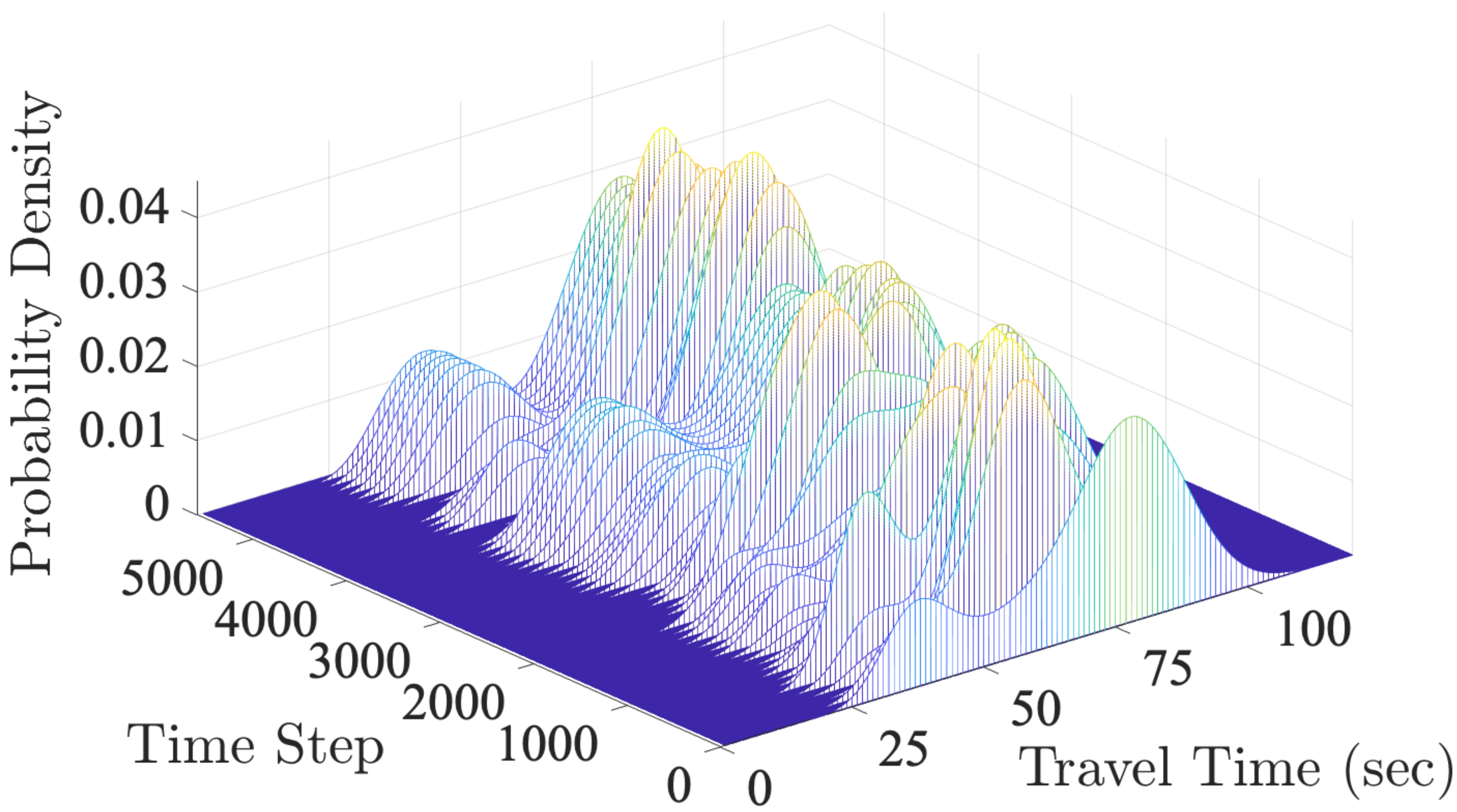}}	
	\caption{Recursive estimation of travel time densities: a snapshot of density evolution in real-time; dataset: I-80.} \label{f11}
	
\end{figure}

\section{Conclusions}
\label{sec:conclusion}
We have introduced an efficient model-based approach for  estimating travel time distributions in urban networks. Our  methods employ sparse model selection on a  mixture density to obtain parsimonious estimates that accurately characterize measured histograms of travel times.  The numerical examples employed in the paper demonstrate that the proposed approach is a viable alternative to existing density estimation techniques and yields estimates with (i) higher goodness-of-fit, (ii) substantial compression compared to the Parzen estimates (i.e., the histogram), and (iii) robustness to over-fitting. 

In this sparsity-seeking framework, ensuring integrability of the mixture densities (i.e., ensuring that the resulting function is a PDF) cannot be achieved by normalization as is traditionally done.  For this purpose, we have developed a new mixture using Mittag-Leffler functions which was shown to outperform Gaussian mixtures in terms of both accuracy and parsimony.  

Most learning algorithms, including sparse model selection, are naturally \emph{offline} in the sense that they operate on the entirety of a given dataset. To address the crucial problem of \emph{online travel time estimation}, we have proposed algorithms that directly operate on streaming data measurements in two settings: (i) successively improving the fitting fidelity when new data become available and (ii) tracking the variability of the travel times in real-time. Our experiments demonstrate a speed-up of several orders of magnitude over offline data analysis. 


\section*{Acknowledgment}
This research was funded in part by the NYU Global Seed Grant for Collaborative Research.  The work of the second author, while with NYU Abu Dhabi and NYU Tandon School of Engineering, was supported by the National Science Foundation (NSF) under grant CCF-1717207.

\appendix
\gdef\thesection{Appendix \Alph{section}}

\section{Notation} \label{A:notation}
\begin{longtable}{c p{0.85\linewidth}}
	\multicolumn{2}{c}{\textbf{General}} \\
	$\mathbb{Z}_+$, $\RR_+$ & The non-negative integers and non-negative real numbers, respectively \\
	$i$ & The imaginary unit, $i \equiv \sqrt{-1}$ \\
	$\mathbb{1}_{\{\mathsf{conditon}\}}$ & The indicator function, maps to 1 if $\mathsf{conditon}$ is true and maps to 0 otherwise \\
	$B$ & The Beta function \\
	$\Gamma$ & The Gamma function \\
	$E_{\nu}$ & Mittag-Leffler function with parameter $\nu$, $E_{\nu}(t) = \sum_{n=0}^{\infty} \frac{t^n}{\Gamma(1 + n\nu)}$ \\
	$\|\cdot \|$ & Norm, e.g., $\| y \|_2$ is the $L_2$ norm of $y$ \\
	& \\
	\multicolumn{2}{c}{\textbf{Traffic-Flow}} \\
	$C(x)$ & Vehicle crossing time at position $x$ \\
	$\Pi(x)$ & Macroscopic pace at position $x$ \\
	$\rho$ & Traffic density \\
	$V$ & Equilibrium speed function \\
	$Q$ & Equilibrium flux function (fundamental diagram), $Q(\rho) = \rho V(\rho)$ \\
	$P$ & Equilibrium pace function, which maps traffic density to pace, $P(\rho) = 1 / V(\rho)$ \\
	$\rho_{\mathrm{jam}}$ & Jammed traffic density, $V(\rho_{\mathrm{jam}}) = 0$ \\
	$v_{\mathrm{fr}}$ & Free-flow speed, $V(0) = v_{\mathrm{fr}}$ \\
	$v_{\mathrm{b}}$ & Backward wave speed, $\frac{\dd}{\dd \rho}Q(\rho_{\mathrm{jam}}) = v_{\mathrm{b}}$ \\
	& \\
	\multicolumn{2}{c}{\textbf{Probabilities and Related Notions}} \\
	$\EE$ & The expectation operator \\
	$\binom{m}{k_1 ~ \hdots ~ k_j}$ & The multinomial coefficient, $\binom{m}{k_1 ~ \hdots ~ k_j} \equiv \frac{m!}{k_1! \cdot \hdots \cdot k_j!}$ \\
	$f_X$ & The probability density function (PDF) associated with \textit{continuous} random variable $X$ \\
	$p_X$ & The probability mass function (PMF) associated with \textit{discrete} random variable $X$ \\
	$\varphi_X$ & The characteristic function associated with random variable $X$, $\varphi_X(s) \equiv \EE e^{isX}$ \\
	$f(y; \mu)$ & A PDF evaluated at $y$ with parameter (vector) $\mu$.  We casually write $f(y)$ ignoring the argument $\mu$, so as to lighten notation. \\
	$p(y; \mu)$ & A PMF evaluated at $y$ with parameter (vector) $\mu$.  We casually write $p(y)$ ignoring the argument $\mu$, so as to lighten notation. \\
	$f_{\Gamma}$ & The PDF of a Gamma distributed random variable \\
	$f_{\mathrm{M-L}}$ & Mittag-Leffler PDF, $f_{\mathrm{M-L}}(t;\beta,\sigma,b) = \frac{1}{\sigma \Gamma(\beta)} \Big(\frac{t}{\sigma}\Big)^{\beta - 1}  \Big[ E_{\nu} \Big( \big(\frac{t}{\sigma}\big)^c \Big) \Big]^{-1}$ \\
	$f_{\rho}$ & The PDF of traffic density \\
	$f_{\Pi}$ & The PDF of pace \\
	$f_{T}$ & The PDF of travel time \\
	$\varphi_{\Pi}$ & The characteristic function of pace \\
	$\varphi_{\Gamma}$ & The characteristic function of a Gamma distributed random variable \\
	& \\
	\multicolumn{2}{c}{\textbf{Travel Time Data, Empirical Distributions, and Estimated Distributions}} \\
	$w$ & Regularization parameter \\
	$T_1, \hdots, T_S$ & A sample of $S$ travel times \\
	$\mathbb{T}_W^{(j)}$ & The $j$th window taken from stream travel time data, $\mathbb{T}_W^{(j)} = \{T_j, \hdots, T_{j + W-1}\}$, where $W$ is the window width \\
	$\Delta$ & Time discretization constant \\
	$\{\tau_n\}_{n=0}^{N-1}$ & A set of discrete travel times, defining the support (or domain) of the PMFs, $t_n = n\Delta$ \\
	$\{t_m\}_{m=0}^{M-1}$ & A subset of $\{\tau_n\}_{n=0}^{N-1}$ representing the locations of the mixture PDFs \\
	$\tau$ & A function that maps a continuous travel time $t$ to a discrete travel $\tau_n$ ($\tau_n$ is the representative of $t$ in the discrete set) \\
	$\widehat{f}$ & Empirical distribution of travel times, also known as the Parzen Window (PW) estimator. $\widehat{f}(t)$ is the frequency of travel time $t \in \RR_+$, as established empirically. \\
	$k_h$ & A kernel, window, or bin of width $h$;  $k_h(t - T_j)$ provides a measure of the distance between travel time $t \in \RR_+$ and the data point $T_j$. \\
	$\overline{f}$ & Mixture PDF (to be estimated) \\
	$\phi_m$ & The $m$ mixture component of $\overline{f}$ (a PDF) \\
	$\theta_m$ & The weight associated with the $m$th mixture component \\
	$\theta$ & $M$ dimensional vector of mixture weights \\
	$\widehat{p}$ & Discrete empirical distribution, an $N$ dimensional vector $\widehat{p} = [\widehat{p}_0 ~ \hdots ~ \widehat{p}_{N-1}]^{\top}$ with $\widehat{p}_n \propto \widehat{f}(\tau_n)$ \\
	$\overline{p}$ & Discrete mixture distribution, an $N$ dimensional vector $\overline{p} = [\overline{p}_0 ~ \hdots ~ \overline{p}_{N-1}]^{\top}$ with $\overline{p}_n \propto \overline{f}(\tau_n)$ \\
	$c_n$ & Discretization constant used to define the discrete component densities; specifically, $\phi_{n,m} \equiv c_n \phi_m(\tau_n)$ \\
	$\Phi$ & The matrix $[\phi_{n,m}] \in \RR_+^{N\times M}$: we have $\overline{p} = \Phi \theta$ \\
	$\Psi$ & A $N\times N$ matrix with elements $\Psi_{n,i} = \kappa_h(\tau_n - \tau_i)$ \\
	$\widehat{f}^{(K)}$ & Parzen density established using the fist $K$ travel times $\{T_1, \hdots, T_K\}$ \\
	$\widehat{f}^{(j)}$ & Parzen density established using travel time data $\mathbb{T}_W^{(j)}$ \\
	$\widehat{p}^{(K)}$ & Discrete empirical distribution established using the fist $K$ travel times $\{T_1, \hdots, T_K\}$ \\
	$\widehat{p}^{(j)}$ & Discrete empirical distribution established using travel time data $\mathbb{T}_W^{(j)}$ \\
	$\widehat{\theta}^{(K)}$ & Estimated mixture weights using the first $K$ travel times $\{T_1, \hdots, T_K\}$ \\
	$\widehat{\theta}^{(j)}$ & Estimated mixture weights using travel time data $\mathbb{T}_W^{(j)}$
\end{longtable}

\section{Ensuring Summability of Mixture Weights to Unity} \label{A:summability}
Let $[\theta^*_0 \hdots \theta^*_{M-1}]^{\top}$ solve \eqref{eq_LASSO}, where (for the sake of generality) the matrix elements $\phi_{n,m}$ are given by \eqref{eq_genGammaDen}.  Suppose $\sum_{m=0}^{M-1} \theta^*_m < 1$
.  To address the summability to unity issue, we may append a single component density to the solution with negligible impact on the outcome.  Consider the vector $\psi(t^{\prime},\sigma^{\prime})  \in \RR^N$, the elements of which are given by
\begin{align}
\psi_n(t^{\prime},\sigma^{\prime})  \equiv \sigma^{\prime} f_{\mathrm{M-L}} \Big(t^{\prime}; 1+\frac{\tau_n}{\sigma^{\prime}}, \sigma^{\prime} \Big),
\end{align}
for $n=0,1,\hdots N-1.$
The parameters $t^{\prime}$ and $\sigma^{\prime}$ are chosen so that 
\begin{align}
\underset{n \in \{0,\hdots,N-1\}}{\max} ~ \psi_n(t^{\prime},\sigma^{\prime}) \le \varepsilon, \label{eq_mode}
\end{align}
for some predefined tolerance threshold $\varepsilon > 0$. 
Define $\Delta^{\prime} \equiv \frac{\Delta}{\sigma^{\prime}}$ so that $\psi_n(t^{\prime},\sigma^{\prime})  = \sigma^{\prime} f_{\mathrm{M-L}} (t^{\prime}; 1+n\Delta^{\prime}, \sigma^{\prime} )$. 
Consider a choice of $t^{\prime}$ and $\sigma^{\prime}$ so that $\frac{t^{\prime}}{\sigma^{\prime}} = 1$,  then
\begin{align}
\underset{n \in \{0,\hdots,N-1\}}{\max} ~ \sigma^{\prime} f_{\mathrm{M-L}} (t^{\prime}; 1+n\Delta^{\prime}, \sigma^{\prime} ) = \underset{n \in \{0,\hdots,N-1\}}{\max} ~ \frac{1}{\Gamma(1 + n\Delta^{\prime}) E_{\Delta^{\prime}}(1) }. 
\end{align}
A well-known property of the Gamma function is that it achieves a global minimum in $\RR_+$, which is $\Gamma(x_{\min}) = 0.885603$ (for $x_{\min} = 1.461632$). 
%
%
Consequently,  
\begin{align}
\underset{n \in \{0,\hdots,N-1\}}{\max} \psi_{n}(t^{\prime},\sigma^{\prime}) \le \frac{1}{0.88 E_{\Delta^{\prime}}(1)}
\end{align}
so that $\sigma^{\prime}$ is chosen to ensure that
\begin{align}
E_{\frac{\Delta}{\sigma^{\prime}}}(1) \ge \frac{1}{0.88 \varepsilon}.
\end{align}
We now append $\psi(t^{\prime},\sigma^{\prime})$ to $\Phi$ (as a column to the right) and set $\theta = [\theta^{*\top}, ~~ 1-\sum_{m=0}^{M-1} \theta_m^* ]^{\top}$.
First, notice that the choice of $\sigma^{\prime}$ above does not depend on $\theta^*$ (but depends exclusively on the discretization interval $\Delta$).  Therefore, this calculation can be performed offline.  
Since $\theta_M = 1-\sum_{n=0}^{M-1} \theta_m^*  < 1$, by design, we know that the contribution of $\psi(t^{\prime},\sigma^{\prime})$  to $\overline{p}$ is smaller than $\varepsilon$ (since its  contribution to all support values $\{\tau_n\}_{n=0}^{N-1}$ is smaller than $\varepsilon$). 
This motivates restricting attention to LASSO constrained to the positive orthant (vs. the probability simplex).

\section{Increasing the Sparsity} \label{A:sparsity}
In the adaptive case, we can further increase the sparsity by scaling the weights $\theta$ in a way that favors mixture components with larger scale parameters, as a type of \emph{preconditioning}. Formally, let $\varSigma \in \RR^{M \times M}$ be a diagonal matrix with elements $\varSigma_{m,m} = \sigma_m$ and consider the following re-scaled version of the estimation problem: 
\begin{equation}
\underset{\theta \in \Omega}{\mathrm{minimize}} ~~ \frac{1}{2} \big\lVert \widehat{p} - \Phi \theta \big\rVert_2^2 + w \big{\lVert} \varSigma^{-1} \theta \big{\rVert_1}. \label{eq_adaptiveProb_mod}
\end{equation}
By modifying the weight vector in this way, we  penalize each mixture component in proportion to the inverse of its scale parameter. This is done to  encourage the sparse density algorithm to choose  mixture components with larger scale parameters (hence, fewer components) to capture the distribution.  Informally, when two or more mixture components yield a fitting accuracy comparable with one wider component, the latter will be selected. 

When $\Omega = \RR_+^M$, we have the constrained LASSO problem
\begin{equation}
\underset{\theta \in \RR^{M+1}_+}{\mathrm{minimize}} ~~ \frac{1}{2} \big\lVert \widehat{p} - \Phi \theta \big\rVert_2^2 + w \mathbf{1}^\top \varSigma^{-1} \theta,
\end{equation}
where a logarithmic barrier for the non-negative constraints can be augmented to the objective function to obtain the associated centering problem 
\begin{equation}
\underset{\theta \in \RR^{M+1}}{\mathrm{minimize}} ~~ \frac{z}{2} \big\lVert \widehat{p} - \Phi \theta \big\rVert_2^2 + zw \sum_{m=0}^{M} \frac{1}{\sigma_m} \theta_m  -\sum_{m=0}^{M}\log(\theta_i).
\end{equation}

\section{Post-Processing}\label{A:postProcessing}
Once a numerical solution of LASSO~\eqref{eq:consLASSO} is obtained, it is important to numerically \emph{post-process} it. For example, we aim for `zero' values in the solution vector $\theta$, but this practically corresponds to very small entries. One simple yet effective way to define the zero entries of $\theta$ is by thresholding, e.g., setting all entries $|\theta_i| < \epsilon \|\theta\|_{\infty}$ to zero, for some small value of $\epsilon$, e.g. $\epsilon = 10^{-3}$.  After thresholding, the support of the solution $\supp (\widehat{\theta})$ (the set of non-zero entries) and the corresponding number of non-zero entries $s_w\equiv |\supp(\theta)|$ are determined. An additional way to improve sparsity is by combining nearby mixture components that appear in the (thresholded) solution, i.e., mixture components whose locations lie within a predetermined distance. Finally, we may improve the reconstruction fidelity by performing constrained least-squares on the resulting support: i.e., we obtain a new matrix $\Phi_s\in\RR^{N\times s_w}$ by selecting the set of columns of $\Phi$ corresponding to the  support, and perform constrained least-squares to update the entries $\widehat{\theta}_s$:
\begin{equation}\label{eq:debiasing}
\underset{\theta_s \in \Omega_s}{\mathrm{minimize}} ~~ \big\lVert \widehat{p} - \Phi_s \theta_s \big\rVert_2^2.
\end{equation}
This is usually referred to as a \emph{de-biasing} step, where $\Omega_s = \RR^{s_w}_+$.

\section{Choice of Regularization Parameter} \label{A:regularization}
The regularization parameter $w$ controls the trade-off between sparsity and reconstruction error. If the regularization parameter $w$ is sufficiently large most of the coefficients are driven to zero, thus leading to a sparse model with only a few (relevant) mixture density functions. However, this typically leads to poor fitting accuracy (low goodness-of-fit). On the other hand, when $w$ is sufficiently small one retrieves the best possible fit (non-negative least-squares), which is (in general) not sparse: most (typically, all) coefficients are non-zero.  In selecting $w$, the aim is to balance the trade-off between goodness-of-fit and sparsity. The problem of choosing the appropriate regularization parameter is crucial as it governs the selection of the sparsest model that can faithfully reconstruct the underlying distribution of the data.  One approach to select a suitable $w$, which makes good use of the available dataset, is $k$-fold cross-validation \citep{efron1983leisurely,turney1994theory}. Notwithstanding, cross-validation techniques do not promote sparsity in general, but are rather geared towards avoiding \textit{overfitting}. Moreover, an issue with  cross-validation is that it does not lead to consistent model selection for LASSO. 

We propose a simple scheme for tuning the parameter $w$ to balance the trade-off  between goodness-of-fit and sparsity.  For this purpose, we use a metric inspired by the analysis in \citep{reid2013study, sun2012scaled} on \emph{scaled-LASSO}, namely
\begin{equation}
S^2_w \equiv \frac{\|\widehat{p} - \Phi \theta(w) \|_2^2}{M-s_w}, \label{eq_metric}
\end{equation}
where $s_w \equiv |\supp(\theta(w))|$ is the cardinality of the support set (as determined via the post-processing mechanism in \ref{A:postProcessing}), i.e., the number of non-zero entries of the solution vector. We use $\theta(w)$ to emphasize the dependence of the (constrained) LASSO solution on the regularizing parameter $w$. The metric $S^2_w$ in \eqref{eq_metric} captures the trade-off between (i) goodness-of-fit, as measured by the squared $\ell_2-$error $\|\widehat{p} - \Phi \theta(w) \|_2^2$ and (ii) sparsity $(M-s_w)$ (the number of zeros in the solution $\theta$): it is proportional to the former and inversely proportional to the latter. Note, therefore, that seeking to minimize this metric leads to aiming for \emph{simultaneously maximizing the goodness-of-fit and parsimony}, and this is exactly the approach that we adopt in this paper.  
Last, $S^2_w$ is well-defined for $s_w <M$, i.e., it is not defined for values of $w$ close to 0 
where typically $s_w=M$ ($S^2_w$ is finite on a set $(w_{\min},+\infty)$ for some $w_{\min}>0$  because of the continuity of the optimal solution $\theta$ in $w$); we may extend it to take the value infinity in such case (since a sparse solution is desirable).

For $w = 0$, one retrieves the constrained least-squares solution: 
\begin{equation}
\underset{\theta \in \Omega}{\mathrm{minimize}} ~~ \big\lVert \widehat{p} - \Phi \theta \big\rVert_2^2, 
\end{equation}
which serves as a lower bound for the squared $\ell_2-$error (best possible goodness-of-fit) but is known to be non-sparse ($s_{w} = M$ in most cases). For $w > w_0$ where
\begin{equation}
w_0 \equiv \|\Phi^{\top} \widehat{p}\|_{\infty},
\end{equation}
the all-zero solution is retrieved ($s_w = 0$); this maximizes sparsity but yields a squared $\ell_2-$error equal to $\|\widehat{p}\|_2^2$.
One may then search over variable values of $w$ and select the one that \emph{minimizes} $S_w^2$. For example, we may consider values of $w$ in a logarithmic scale: starting from $w_0$ we evaluate  $S_w^2$ for values $w_k =  \eta^k w_0$ for some $\eta \in (0,1)$, e.g., $\eta = 0.95$ was chosen in our experiments, where $k$ is successively increased until a termination criterion is met. In our experiments we have considered: 
\begin{align}
\frac{\|\widehat{p} - \Phi \theta(w_{k+1}) \|_2 - \|\widehat{p} - \Phi \theta(w_k) \|_2 }{\|\widehat{p} - \Phi \theta(w_k) \|_2} < \epsilon',
\end{align} 
with $\epsilon'=10^{-3}$.  An alternative is to achieve a desirable sparsity level \emph{exactly} by means of the search mechanism above in conjunction with bisection. This can be applied to all the sparse estimation problems that we consider in this paper (see Figure \ref{f2_R2} for illustration).

	
	
	
\bibliographystyle{plainnat}
\bibliography{refs.bib}
	
	
	
	
	
	

\end{document}